\documentclass[11pt]{article}

\usepackage[
  shownumpages,               % comment to hide page numbers
  bgcolor={245,245,250},      % box bg color: RGB or xcolor name (e.g. gray, blue, etc.)
  braincolor={190,32,240},    % accent color: RGB or xcolor name (e.g. gray, blue, etc.)
  citingstyle=authoryear,     % authoryear: [Ivanov et al., 2023] | numbers: [1] | super: ¹
  bibliostyle=plainnat,       % plainnat, abbrvnat, unsrtnat, ...
  bibfile=references          % .bib file name (without .bib)
]{brainlab}

\usepackage{url}            % simple URL typesetting
\usepackage{booktabs}       % professional-quality tables
\usepackage{amsfonts}       % blackboard math symbols
\usepackage{nicefrac}       % compact symbols for 1/2, etc.
\usepackage{microtype}      % microtypography
\usepackage{xcolor}         % colors

\usepackage{amsmath}
\usepackage{amssymb}
\usepackage{mathtools}
\usepackage{amsthm}
\usepackage{wrapfig}
\usepackage{nicefrac}
\allowdisplaybreaks

\usepackage{pifont}

\usepackage{algorithm}
\usepackage{algpseudocode}
\usepackage{multirow}

\newcommand{\expect}[1]{\mathbb{E}\left[ #1 \right]}
\newcommand{\norm}[1]{\left\| #1 \right\|}
\newcommand{\expectBig}[1]{\mathbb{E}\Big[ #1 \Big]}
\newcommand{\normBig}[1]{\Big\| #1 \Big\|}
\newcommand{\EEb}[2]{\mathbb{E}_{#1}\left[ #2 \right]}
\newcommand{\dotprod}[2]{\left\langle #1,#2 \right\rangle}

\newcommand{\circledOne}{\text{\ding{172}}}
\newcommand{\circledTwo}{\text{\ding{173}}}
\newcommand{\circledThree}{\text{\ding{174}}}
\newcommand{\circledFour}{\text{\ding{175}}}  
\newcommand{\circledFive}{\text{\ding{176}}} 
\newcommand{\circledSix}{\text{\ding{177}}}

\usepackage{graphicx}
\usepackage{subfigure}
\usepackage{multirow}
\usepackage{colortbl}
\definecolor{bgcolor}{rgb}{0.8,1,1}
\definecolor{bgcolor2}{rgb}{0.8,1,0.8}
\usepackage{threeparttable}

\usepackage{tcolorbox}
\usepackage{pifont}
\definecolor{mydarkgreen}{RGB}{39,130,67}
\definecolor{mydarkred}{RGB}{192,47,25}
\newcommand{\green}{\color{mydarkgreen}}
\newcommand{\red}{\color{mydarkred}}
\newcommand{\cmark}{{\green\ding{51}}}
\newcommand{\xmark}{{\red\ding{55}}}
\usepackage[font=small]{caption}
\usepackage{subcaption}
\usepackage{enumitem}

\newcommand{\paramr}[1]{\textcolor{mydarkred}{#1}}
\newcommand{\paramg}[1]{\textcolor{mydarkgreen}{#1}}

\setbrainmeta{
  title={Leveraging Coordinate Momentum in SignSGD and Muon: Memory-Optimized Zero-Order
LLM Fine-Tuning},
  authors={
    Egor Petrov\textsuperscript{1}, Grigoriy Evseev\textsuperscript{1}, Aleksey Antonov\textsuperscript{1, 2}, Andrey Veprikov\textsuperscript{1},Nikolay Bushkov\textsuperscript{2}, Stanislav Moiseev\textsuperscript{2}, Aleksandr Beznosikov\textsuperscript{1, 3}
  },
  affiliations={
    \textsuperscript{1}Basic Research of Artificial Intelligence Laboratory (BRAIn Lab) \\
    \textsuperscript{2}T-Technologies \\
    \textsuperscript{3}Innopolis University
  },
  abstract={
    Fine-tuning Large Language Models (LLMs) is essential for adapting pre-trained models to downstream tasks. Yet traditional first-order optimizers such as Stochastic Gradient Descent (SGD) and Adam incur prohibitive memory and computational costs that scale poorly with model size. In this paper, we investigate zero-order (ZO) optimization methods as a memory- and compute-efficient alternative, particularly in the context of parameter-efficient fine-tuning techniques like LoRA. We propose \texttt{JAGUAR SignSGD}, a ZO momentum-based algorithm that extends ZO SignSGD, requiring the same number of parameters as the standard ZO SGD and only $\mathcal{O}(1)$ function evaluations per iteration. To the best of our knowledge, this is the first study to establish rigorous convergence guarantees for SignSGD in the stochastic ZO case. We further propose \texttt{JAGUAR Muon}, a novel ZO extension of the Muon optimizer that leverages the matrix structure of model parameters, and we provide its convergence rate under arbitrary stochastic noise. Through extensive experiments on challenging LLM fine-tuning benchmarks, we demonstrate that the proposed algorithms meet or exceed the convergence quality of standard first-order methods, achieving significant memory reduction. Our theoretical and empirical results establish new ZO optimization methods as a practical and theoretically grounded approach for resource-constrained LLM adaptation. Our code is available at \url{https://github.com/brain-lab-research/zero-order-optimization}
  },
}

\begin{document}

% \maketitle

\begin{mainpart}

% \begin{abstract}
  % Fine-tuning large language models (LLMs) is a critical step in adapting pre-trained models to specific tasks, but the computational cost of traditional gradient-based optimization methods, such as stochastic gradient descent (SGD) and its variants, grows significantly with model size. 
  % In this work, we explore zero-order (ZO) optimization as an efficient alternative for fine-tuning LLMs, particularly in the context of parameter-efficient fine-tuning (PEFT) methods like LoRA. We focus on Zero-Order approximations of the gradient, which eliminate the need for costly backpropagation, and propose novel enhancements to improve memory and computational efficiency. Specifically, we introduce the matrix-efficient JAGUAR to incorporate momentum into ZO-based SGD and Sign-SGD, achieving performance gains at negligible cost. Additionally, we introduce the ZO-MUON method, containing key features of zero-order optimization and proven to be effective MUON algorithm. Moreover, we provide a smart sampling strategy for Gaussian matrices in ZO gradient estimation that provides a way to avoid the iterative Newton-Schultz procedure performed at each iteration of the conventional method, due to the zero-order specificity. Our empirical results demonstrate that these techniques significantly improve competitive performance, while maintaining computational efficiency. This work advances the field of ZO optimization for LLMs, offering practical and theoretically grounded methods for resource-constrained environments.
  
% \end{abstract}

\section{Introduction} \label{sec:intro}
Fine-tuning pre-trained Large Language Models (LLMs) has become the standard technique in modern natural language processing~\cite{howard2018universal,zhang2019dialogpt,zhang2024selective,lester2021power}, enabling rapid adaptation to diverse downstream tasks with minimal labeled data~\cite{raffel2020exploring, sanh2021multitask,zaken2021bitfit}. These models, often trained on massive corpora, achieve state-of-the-art results when fine-tuned on specific applications, including question answering, summarization, and dialogue generation. The fine-tuning setup can be considered as a stochastic unconstrained optimization problem of the form 
\begin{equation}
\label{eq:main_problem}
    f^* := \min_{x \in \mathbb{R}^d} \left\{ f(x) := \EEb{\xi \sim \mathcal{D}}{f(x, \xi)} \right\},
\end{equation}
where $x$ are parameters of the fine-tuned LLM, $\mathcal{D}$ is the data distribution available for training, and $f(x, \xi)$ is the loss on data point $\xi$. 

The de facto standard for solving \eqref{eq:main_problem} is the use of First-Order (FO) optimization methods. These approaches assume access to the stochastic gradient \(\nabla f(x, \xi)\). Classical FO methods, such as Stochastic Gradient Descent (SGD) \cite{amari1993backpropagation} and Adam \cite{kingma2014adam}, remain the most widely used techniques for model adaptation due to their efficiency and compatibility with the backpropagation algorithm. Nevertheless, in contemporary fine-tuning tasks, alternative FO algorithms are often preferred.

A recent trend in optimization for LLMs is to represent optimization parameters in matrix form rather than as vectors \cite{bernstein2024old, bernstein2024modular, pethick2025training}. Algorithms such as Shampoo \cite{Shampoo} and SOAP \cite{vyas2024soap} have demonstrated superior performance on LLM training tasks compared to Adam and SGD \cite{dahl2023benchmarking}, which operate in an element-wise manner and do not utilize the underlying structure of the model parameters. Currently, the canonical matrix-based optimization algorithm is Muon \cite{muon_base, muon_scaled, muon_convergence}, which integrates the principles of Shampoo and SOAP but does not employ any preconditioning matrices \cite{muon_base}. The central idea of this method is to project the gradient at each iteration onto the space of semi-orthogonal matrices using the Newton–Schultz algorithm \cite{bernstein2024old}.

%SignSGD has recently emerged as a standard optimization method for fine-tuning LLMs

However, as LLMs continue to scale, the backpropagation procedure, necessary for FO methods, becomes increasingly expensive in terms of memory consumption. For instance, the memory cost of computing gradients during the training of OPT-13B is reported to be more than an order of magnitude larger than that of inference \cite{zhu2023efficient}. This imbalance poses a serious bottleneck for deploying LLM fine-tuning in resource-constrained environments such as edge devices \cite{zhu2023pockengine, gao2024enabling}, consumer-grade GPUs \cite{liao2024lohan, yin2023modulora}, or large-scale distributed settings \cite{han2015deep}.
To overcome these limitations, researchers are exploring various approaches to reduce the size of the required optimizer statistics. One such approach is the SignSGD algorithm, initially developed for distributed optimization \cite{yang2020adaptive}, but which has also proven effective in LLM fine-tuning \cite{pengsoftsignsgd}, owing to its simplicity, memory efficiency, and surprising empirical effectiveness across a range of adaptation tasks \cite{jin2020stochastic,mengoli2025byzantine}. SignSGD was first rigorously analyzed in the FO setting by \cite{bernstein2018signsgd} and \cite{balles2017dissecting}. 
Minimal memory usage and straightforward hyperparameter tuning make SignSGD an attractive choice for memory-constrained fine-tuning of LLMs ($\sim4/3\times$ memory usage compared to Adam). 
Beyond SignSGD, other FO methods also target memory reduction. AdaFactor \cite{shazeer2018adafactor} was among the first, lowering memory usage by storing a single value per block ($\sim4/3\times$). Additional techniques include quantizing optimizer states to lower-precision formats \cite{dettmers20218, li2023memory} ($\sim4/3\times$ and $\sim16/9\times$ respectively) and fusing the backward pass with optimizer updates \cite{lv2023full} ($\sim4/3\times$), further decreasing memory demands during training.

Nevertheless, the most memory-efficient methods are based on the Zero-Order (ZO) optimization technique, which avoids backpropagation entirely by estimating gradients using only forward passes. 
%ZO methods are particularly attractive for LLM fine-tuning because they decouple optimization from the internal structure of the model. 
This flexibility allows us to treat the model as a black box, optimizing performance with minimal assumptions about its architecture or implementation details. Recent studies \cite{malladi2023mezo} have demonstrated the practical benefits of this approach: for example, the MeZO algorithm applies classical ZO SGD \cite{ghadimi2013stochastic} to fine-tune LLMs while maintaining four times lower memory requirements than traditional FO methods~\cite{malladi2023fine} ($\sim10\times$ compared to Adam~\cite{zo_bench}). In ZO methods it is assumed that we only have access to the values of the stochastic function $f(x, \xi)$ from \eqref{eq:main_problem} \cite{flaxman2005online,ghadimi2013stochastic}. Within LLMs pretraining or fine-tuning context, oracles are forward passes with small perturbations in parameters of the model. To estimate gradients, authors use finite differences: 
\begin{equation}
\label{eq:zo_grad}
    \nabla f(x, \xi) \approx \widetilde{\nabla} f(x, \xi) = \frac{f(x + \tau e, \xi) - f(x - \tau e, \xi)}{2\tau} e,
\end{equation}
where $ \tau > 0 $ is a small number, frequently referred to as a smoothing parameter, and $e \in \mathbb{R}^d$ is some random vector \cite{nesterov2017random,duchi2015optimal, malladi2023fine, zo_bench}. In the next section, we provide review about different ZO optimization methods, that somehow utilize formula \eqref{eq:zo_grad}.

%\textbf{Outline.} The rest of the paper is organized as follows. In Section  \ref{sec:rw}, we review related work, categorizing existing research into key topics and highlighting their main contributions. In Section \ref{sec:main}, we derive theoretical results for the introduced methods. Section \ref{sec:exp} provides an empirical study validating these theoretical results. Sections \ref{sec:disc} and \ref{sec:concl} discuss and summarize our findings, offering insights and conclusions. Additional experiments and proofs of theorems are included in Appendix \ref{app:A}, \ref{app:B}.

\section{Related Work and Our Contributions} \label{sec:rw}

\textbf{ZO gradient estimators.} 
The simplest zero-order gradient estimator employs the estimate \eqref{eq:zo_grad} as the stochastic gradient. However, even this approach presents specific challenges, particularly regarding the selection of an appropriate distribution from which to sample the random vector $e$.
The most commonly employed distributions include a uniform sampling over the unit sphere: $e \sim RS(1)^d_{\| \cdot \|}$ \cite{flaxman2005online,nesterov2017random}, a Gaussian distribution with zero mean and identity covariance matrix: $e \sim \mathcal{N}(0, I)$ \cite{nesterov2017random,ghadimi2013stochastic}, and standard basis one-hot vectors \cite{duchi2015optimal,shamir2013complexity}. 
Also, some papers \cite{lian2016comprehensive, sahu2019towards, akhtar2022zeroth} utilize the so-called full coordinate estimate, which approximates the gradient across all basis vectors. However, this approach requires $\mathcal{O}(d)$ calls to the zero-order oracle, making it impractical for large-scale fine-tuning tasks.
Despite the prevalence of these approaches, alternative and more complicated sampling strategies have also been explored.

In \cite{roberts2023direct, nozawa2025zero}, the authors explore low-dimensional perturbations within random subspaces. The central concept of random subspace methods involves generating the perturbation vector $e$ within a subspace spanned by a projection matrix $P \in \mathbb{R}^{d \times r}$ and a low-dimensional random vector $\tilde{e} \in \mathbb{R}^r$: $e = P \tilde{e}$. Typically, $P$ and $\tilde{e}$ are sampled from a Gaussian distribution and $r \ll d$. The primary motivation for this method lies in the fact that gradients during the fine-tuning process exhibit a low-dimensional structure \cite{nozawa2025zero}.
In \cite{liu2024sparse, wang2024simultaneous}, the authors employ a masked random vector $e$, wherein at each iteration a random mask with $r$ non-zero elements $m_r \in \{0,1\}^d$ is generated and applied element-wise to a Gaussian vector $e$. This procedure accelerates the optimization step, as only the parameters corresponding to the active entries in $m_r$ are updated, rather than the entire parameter set. In contrast, the authors of \cite{extreme_sparsity} depart from random mask sampling at each iteration and instead select an optimal mask $m_r$ prior to training, according to a specific criterion. Consequently, the update rule \eqref{eq:zo_grad} modifies only the parameters selected by the optimal mask during optimization.
In our approach, we similarly do not utilize all coordinates of the random vector $e$ in each estimation of \eqref{eq:zo_grad}, instead, we select a single coordinate at each step. However, unlike previous works \cite{liu2024sparse, wang2024simultaneous, extreme_sparsity}, we do not discard information from the remaining coordinates, but accumulate information from previous iterations.
%Authors of \cite{extreme_sparsity} propose an approach wherein only sensitive parameters are trained. Sensitive parameters are defined as $x \odot m_k$, where $\{0,1\}^d \ni m_k = \operatorname{arg}\min\limits_m \Vert m \odot \nabla f(x) \Vert_2^2$ with $k$ non-zero elements. In this setup, only $k$ coordinates are modified within the Stochastic Gradient Descent (SGD) algorithm, utilizing a numerically estimated gradient: $x_{t+1} = x_{t} - \gamma \widetilde{\nabla} f(x_{t} \odot m_k, \xi, e)$. However, the determination of the optimal mask $m_k$ is performed at each optimization step, potentially introducing computational challenges. This issue is solved in \cite{liu2024sparse, wang2024simultaneous}, where authors sample $m_k$ uniformly every iteration.
We employ the JAGUAR zero-order gradient estimation technique \cite{veprikov2024new, nazykov2024stochastic}, which integrates the concept of sampling one-hot basis vectors with the utilization of a SAGA-like momentum update \cite{defazio2014saga}. This approach facilitates convergence in the stochastic setting by leveraging memory from past iterations, while using the same amount of memory as standard zero-order methods like ZO SGD (MeZO) \cite{malladi2023fine}. In the original paper \cite{veprikov2024new}, the authors do not incorporate a momentum parameter, discarding coordinate information from previous iterations. In contrast, we introduce a momentum parameter, $0 \leq \beta \leq 1$ (see Algorithms \ref{algorithm:jaguar} and \ref{algorithm:muon}), which controls the utilization of gradients from past iterations. We demonstrate that adding this momentum $\beta$ allows the method to converge in the stochastic non-convex case (see Theorems \ref{theorem:jaguar_sign} and \ref{theorem:jaguar_muon}).

\textbf{Momentum techniques.} Numerous zero-order methods in the literature incorporate momentum techniques in various forms. However, these approaches typically introduce multiple additional variables of dimension $d$. Since zero-order methods are often chosen for fine-tuning tasks to save memory, the inclusion of such extra variables becomes a critical limitation in these settings.  In \cite{huang2022accelerated}, authors use variance reduction technique SPIDER \cite{fang2018spider}, that uses approximately $5d$ parameters: $2d$ for ZO gradients, $2d$ for model parameters and $1d$ for momentum. In \cite{chen2019zo, jiang2024zo}, the authors employ the Adam optimization technique \cite{kingma2014adam}, which is frequently used for stochastic non-convex optimization problems \cite{chen2019zo,openreview2025mezo_a3dam}. However, this technique incurs a significant memory overhead, requiring $4d$ parameters. The paper \cite{reddy2023convergence} utilizes classical heavy-ball momentum within a zero-order framework, provided, only demonstrating almost sure convergence to a constant in the non-convex setting. In our work, we successfully incorporated a momentum technique using only $2d + 1$ parameters and proved the convergence rate within the standard stochastic non-convex setting (see Algorithm \ref{algorithm:jaguar} and Theorem \ref{theorem:jaguar_sign}). It is worth noting that numerous other zero-order techniques exist in the literature to achieve convergence when the function $f$ is convex \cite{gorbunov2022accelerated,nesterov2017random,duchi2015optimal}, satisfies conditions like PL \cite{reddy2023convergence}  or ABG \cite{rando2024stochastic}, or in deterministic settings \cite{gorbunov2022accelerated}. Since our focus is on fine-tuning problems, which fall under the stochastic non-convex case, we will not discuss these methods in detail. 

% \begin{table}[H]    
%     \centering
%     \caption{Relevant results}
%     \vspace*{0.1in}
%     \label{tab:relevant-results}
%     \resizebox{\linewidth}{!}{
%     \begin{tabular}{lcccccc}
%     \toprule[1pt]
%     \midrule
%     Paper & Year & Reference & Parameter count & Convergence rate & Fine tuning (LLM) & Assumptions \\
%     \midrule
%     Accelerated zero-Order Momentum (Acc-ZOM) & 2022 & [X] & 5d & $\text{O}\left(\text{...}\right)$ (Theorem 3) & - & - \\
%     ZO-AdaMM & 2019 & [Y] & 4d & $\text{O}\left(\frac{d}{\sqrt{T}} + \frac{d^2}{T}\right)$ (Table 1) & - & - \\
%     ZO-ProxSTORM & 2024 & [Z] & 5d & $\text{O}\left(\text{...}\right)$ (Theorem 7) & - & - \\
%     Stochastic Zero-order Heavy Ball & 2023 & [W] & 3d & - & - & - \\
%     ZO-AdaMU & 2024 & [V] & 4d & $\text{O}\left(\text{...}\right)$ (Lemma 4) & + & - \\
%     zero-Order Fine-Tuning of LLMs with Extreme Sparsity & 2024 & \cite{extreme_sparsity} & 3d & O(...) (Theorem 3.3)  & + & ? \\
%     Sparse MEZO & 2024 & \cite{liu2024sparse} & 3d & Theorem 3.4 & + & ? \\
%     Simultaneous Computation and MEZO Optimizer for Fine-Tuning LLMs & 2024 & \cite{wang2024simultaneous} & \textbf{2d} & Lemma 3 & ? \\
%     zero-order Random Subspace Algorithm for Non-Smooth Convex Optimization & 2024 & \cite{}
%     \bottomrule[1pt]
%     \end{tabular}
%     }
% \end{table}

%\subsection{Zero-order optimization benchmarking.}

\textbf{Matrix ZO optimization.}
In the context of zero-order optimization, transitioning to matrix-valued parameters necessitates replacing the random vector $e \in \mathbb{R}^d$ in zero-order gradient approximation \eqref{eq:zo_grad} with a random matrix $E \in \mathbb{R}^{m \times n}$, and correspondingly, projecting this matrix $E$ onto a semi-orthogonal space, as is done in the Muon algorithm \cite{muon_base}. Since the random matrix $E$ is typically drawn from a known distribution, it is possible to directly sample orthogonal matrices when computing the gradient estimator \eqref{eq:zo_grad}. A similar approach has previously appeared in the zero-order optimization literature \cite{lozo}; however, that work did not consider the Muon algorithm, but rather focused on sampling two Gaussian matrices $V \in \mathbb{R}^{m \times r}$ and $U \in \mathbb{R}^{n \times r}$ of rank $r \ll \min\{m, n\}$. This approach does not correspond to the decomposition of the random matrix $E$, as $E$ is almost surely of full rank. Additionally, alternative techniques for sampling low-rank matrices have been proposed in the literature. For instance, in \cite{zo_random_subspaces}, a method analogous to the sampling of low-rank vectors described in \cite{roberts2023direct, nozawa2025zero} is utilized. In our work, we extend our memory-efficient momentum method to the ZO version of the matrix-based Muon algorithm \cite{muon_base} (see Algorithm \ref{algorithm:muon} and Theorem \ref{theorem:jaguar_muon}), keeping the $2d + 1$ parameter efficiency while also broadening our analysis to more modern algorithms that leverage the matrix structure of parameters.

We present a summary of relevant results from the existing zero-order literature in Table \ref{table:zo}.

\begin{table}[h!]  
\renewcommand{\arraystretch}{1.5}
    \centering
%    \small
    \scriptsize
\captionof{table}{Summary of relevant results from the existing zero-order literature.}
% \vspace{-0.2cm}
    \label{tab:comparison0}   
    \scriptsize
\resizebox{\linewidth}{!}{
  \begin{threeparttable}
    \begin{tabular}{|c|c|c|c|c|c|c|}
    \cline{2-6}
    \multicolumn{1}{c|}{} & \textbf{Method} & 
    \textbf{Parameter Count}
    & 
    \begin{tabular}{@{}c@{}}
        \textbf{Convergence Rate}
        \\[-1mm]
        \textbf{Stochastic Non-convex Case}
    \end{tabular}
    &
    \textbf{Momentum}
    &
    \textbf{Fine-tuning (LLM) Setup}
    \\
    \hline
    \multirow{19}{*}{\rotatebox[origin=c]{90}{\shortstack[c]{\textbf{Vector Parameters}\\$\mathbf{x} \boldsymbol{\in} \mathbf{\mathbb{R}}^{\mathbf{d}}$}}} 
    & \texttt{ZO-SGD} \cite{ghadimi2013stochastic}  & \paramg{$\mathbf{2}$} $ \mathbf{\boldsymbol{\cdot} ~ d}$ & \cmark & \xmark & \xmark \\ \cline{2-6}
    & \texttt{ZO-PSGD} \cite{ghadimi2016mini} & \paramg{$\mathbf{2}$} $ \mathbf{\boldsymbol{\cdot} ~ d}$ & \cmark & \xmark & \xmark \\ \cline{2-6}
    & \texttt{ZO-SCD} \cite{lian2016comprehensive} \tnote{{\color{blue} (1)}} & \paramg{$\mathbf{2}$} $ \mathbf{\boldsymbol{\cdot} ~ d}$ & \cmark & \xmark & \xmark \tnote{{\color{blue} (2)}} \\ \cline{2-6}
    & \texttt{ZO-SPIDER} \cite{fang2018spider} & \paramr{$\mathbf{5}$} $ \mathbf{\boldsymbol{\cdot} ~ d}$ & \cmark & \cmark & \xmark \\ \cline{2-6}
    & \texttt{ZO-AdaMM} \cite{chen2019zo} & \paramr{$\mathbf{4}$} $ \mathbf{\boldsymbol{\cdot} ~ d}$ & \cmark & \cmark & \xmark \\ \cline{2-6}
    & \texttt{ZO-SignSGD} \cite{liu2019signsgd} & \paramg{$\mathbf{2}$} $ \mathbf{\boldsymbol{\cdot} ~ d}$ & \xmark~\cmark \tnote{{\color{blue}(3)}} & \xmark & \xmark \tnote{{\color{blue}(4)}} \\ \cline{2-6}
    & \texttt{Acc-ZOM} \cite{huang2022accelerated} & \paramr{$\mathbf{5}$} $ \mathbf{\boldsymbol{\cdot} ~ d}$ & \cmark & \cmark & \xmark \\ \cline{2-6}
    & \texttt{DSFBSD} \cite{roberts2023direct} & \paramg{$\boldsymbol{(}\mathbf{1 \boldsymbol{+}  r}\boldsymbol{)}$} $ \mathbf{\boldsymbol{\cdot} ~ d}$ \tnote{{\color{blue}(5)}} & \xmark & \xmark & \xmark \\ \cline{2-6}
    & \texttt{MeZO} \cite{malladi2023fine} & \paramg{$\mathbf{2}$} $ \mathbf{\boldsymbol{\cdot} ~ d}$ & \xmark & \xmark & \cmark \\ \cline{2-6}
    & \texttt{ZO-ProxSTORM} \cite{qian2023zero} & \paramr{$\mathbf{5}$} $ \mathbf{\boldsymbol{\cdot} ~ d}$ & \cmark & \cmark & \xmark \\ \cline{2-6}
    & \texttt{HB ZO-SGD} \cite{reddy2023convergence} & \paramr{$\mathbf{3}$} $ \mathbf{\boldsymbol{\cdot} ~ d}$ & \xmark \tnote{{\color{blue}(6)}} & \cmark & \xmark \\ \cline{2-6}
    & \texttt{Sparse ZO-SGD} \cite{guo2024zero} & \paramr{$\boldsymbol{(}\mathbf{2 \boldsymbol{+} r}\boldsymbol{)}$} $ \mathbf{\boldsymbol{\cdot} ~ d}$ \tnote{{\color{blue}(5)}} & \xmark & \xmark & \cmark \\ \cline{2-6}
    & \texttt{Sparse MeZO} \cite{liu2024sparse} & \paramr{$\mathbf{3}$} $ \mathbf{\boldsymbol{\cdot} ~ d}$ & \xmark & \xmark & \cmark \\  \cline{2-6}
    & \texttt{LeZO} \cite{wang2024simultaneous} & \paramg{$\mathbf{2}$} $ \mathbf{\boldsymbol{\cdot} ~ d}$ & \xmark & \xmark & \cmark \\ \cline{2-6}
    & \texttt{ZO-AdaMU} \cite{jiang2024zo} & \paramr{$\mathbf{4}$} $ \mathbf{\boldsymbol{\cdot} ~ d}$ & \cmark & \cmark & \cmark \\ \cline{2-6}
    & \texttt{ZO-SGD-Cons} \cite{kim2025conserv} & \paramg{$\mathbf{2}$} $ \mathbf{\boldsymbol{\cdot} ~ d}$ & \xmark & \xmark & \cmark \\ \cline{2-6}
    & \texttt{SGFM} \cite{nozawa2025zero} & \paramr{$\boldsymbol{(}\mathbf{2 \boldsymbol{+} r}\boldsymbol{)}$} $ \mathbf{\boldsymbol{\cdot} ~ d}$ \tnote{{\color{blue}(5)}} & \xmark & \xmark & \xmark \\ \cline{2-6}
    & \texttt{CompSGD} \cite{kornilov2025sign} & \paramg{$\mathbf{2}$} $ \mathbf{\boldsymbol{\cdot} ~ d}$ & \xmark~\cmark \tnote{{\color{blue}(3)}} & \xmark & \cmark \\ \cline{2-6}
    & \cellcolor{bgcolor2}{\begin{tabular}{@{}c@{}}
        \texttt{JAGUAR SignSGD}
        \\
        \textbf{Algorithm} \ref{algorithm:jaguar}
    \end{tabular}} & \cellcolor{bgcolor2}{\paramg{$\mathbf{2}$}} $ \mathbf{\boldsymbol{\cdot} ~ d \boldsymbol{+} 1}$  & \cellcolor{bgcolor2}{\cmark} & \cellcolor{bgcolor2}{\cmark} & \cellcolor{bgcolor2}{\cmark} \\
    \hline
    \hline
    \multirow{6}{*}{\rotatebox[origin=c]{90}{\begin{tabular}{@{}c@{}}
        \textbf{Matrix Parameters}
        \\
        $\mathbf{X} \boldsymbol{\in} \mathbf{\mathbb{R}}^{\mathbf{m} \boldsymbol{\times} \mathbf{n}}$
    \end{tabular}}}
    & \texttt{ZO-RMS} \cite{maass2021zero} \tnote{{\color{blue}(7)}} & \paramg{$\mathbf{2}$} $ \mathbf{\boldsymbol{\cdot} ~ mn}$  & \xmark~\cmark \tnote{{\color{blue}(3)}} & \xmark & \xmark \\ \cline{2-6}
    & \texttt{MeZO} \cite{malladi2023fine} & \paramg{$\mathbf{2}$} $ \mathbf{\boldsymbol{\cdot} ~ mn}$ & \xmark & \xmark & \cmark \\ \cline{2-6}
    & \texttt{LOZO} \cite{lozo} & \paramr{$\boldsymbol{(}\mathbf{m \boldsymbol{+} n\boldsymbol{)} r}$} $\boldsymbol{+}$ \paramg{$\mathbf{2}$} $ \mathbf{\boldsymbol{\cdot} ~ mn}$ \tnote{{\color{blue}(5)}} & \cmark & \xmark & \cmark \\ \cline{2-6}
    & \texttt{SubZero} \cite{zo_random_subspaces} \tnote{{\color{blue}(8)}} & \paramr{$\boldsymbol{(}\mathbf{m \boldsymbol{+} n \boldsymbol{+} r \boldsymbol{)} r}$} $\boldsymbol{+}$ \paramg{$\mathbf{2}$} $ \mathbf{\boldsymbol{\cdot} ~ mn}$ \tnote{{\color{blue}(5)}}  & \xmark & \xmark & \cmark \\ \cline{2-6}
    & \cellcolor{bgcolor2}{\begin{tabular}{@{}c@{}}
        \texttt{JAGUAR Muon}
        \\
        \textbf{Algorithm} \ref{algorithm:muon}
    \end{tabular}} & \cellcolor{bgcolor2}{\paramg{$\mathbf{2}$} $\mathbf{\boldsymbol{\cdot} ~ mn \boldsymbol{+} 1}$}  & \cellcolor{bgcolor2}{\cmark} & \cellcolor{bgcolor2}{\cmark} & \cellcolor{bgcolor2}{\cmark} \\
    \hline
    %%%%%%%%%%%%%%
    \end{tabular}     
    \begin{tablenotes}
    {\small 
        \vspace{1mm}
        \item[] 
        \tnote{{\color{blue}(1)}} Uses a full coordinate ZO estimator.  
        \tnote{{\color{blue}(2)}} Considers asynchronous algorithms.  
        \tnote{{\color{blue}(3)}} Convergence only to a neighborhood of the solution.  
        \tnote{{\color{blue}(4)}} Addresses adversarial attacks in deep learning.  
        \tnote{{\color{blue}(5)}} $r \ll d, m, n$ is a small number.  
        \tnote{{\color{blue}(6)}} Only asymptotic convergence to a constant.  
        \tnote{{\color{blue}(7)}} Assumes that parameters are  symmetric matrices.
        \tnote{{\color{blue}(8)}} Assumes sparsity of parameters.  
    } % "sparse assumption"??
    \end{tablenotes}
    \end{threeparttable}
}
% \vspace{-0.1cm}
\label{table:zo}
\end{table}

\subsection{Our Contributions}
While zero-order optimization methods have recently attracted attention for LLM fine-tuning, previous work has primarily focused on basic algorithms. In this paper, we broaden the scope of zero-order optimization by introducing advanced momentum techniques, specifically adapting the JAGUAR approach \cite{veprikov2024new} to the SignSGD algorithm in the zero-order setting (see Algorithms \ref{algorithm:jaguar}). We consider this algorithm because SignSGD has demonstrated state-of-the-art performance in LLM fine-tuning tasks, outperforming even AdamW \cite{pengsoftsignsgd}. Our key contributions are as follows:
\begin{itemize}[leftmargin=7pt]
    \item We provide the first convergence analysis in the stochastic non-convex setting for zero-order SignSGD with momentum (Algorithm \ref{algorithm:jaguar} and Theorem \ref{theorem:jaguar_sign}), requiring only $2d+1$ parameters and $\mathcal{O}(1)$ ZO oracle calls per iteration.
    \item We extend our memory-efficient momentum method to the Muon algorithm (Algorithm \ref{algorithm:muon}), introducing the first zero-order variant of Muon that preserves memory efficiency. We also establish its convergence rate in the stochastic non-convex setting (Theorem \ref{theorem:jaguar_muon}).
    \item We empirically evaluate the proposed zero-order methods on challenging LLM fine-tuning benchmarks, demonstrating their effectiveness and practical relevance.
\end{itemize}

\section{Main results} \label{sec:main}

\subsection{Preliminaries} \label{subsec:prelim}

\textbf{Notations.} We denote the $\ell_1$ and $\ell_2$ (Euclidean) norms of a vector $x \in \mathbb{R}^d$ as $\|x\|_1 := \sum_{i=1}^d |x_i|$ and $\|x\|_2^2 := \sum_{i=1}^d x_i^2$, respectively. For clarity, matrix-valued variables are denoted by capital letters.
For matrices $X \in \mathbb{R}^{m \times n}$, we use the Schatten 1-norm ($\mathcal{S}_1$) and Schatten 2-norm ($\mathcal{S}_2$, Frobenius): $\|X\|_{\mathcal{S}_1} := \sum_{i=1}^d |(\Sigma_X)_{i,i}|$ and $\|X\|_{\mathcal{S}_2}^2 := \sum_{i=1}^d (\Sigma_X)_{i,i}^2 = \sum_{i=1}^m \sum_{j=1}^n X_{i,j}^2 =: \|X\|_F^2$, where $X = U_X \Sigma_X V_X^T$ is the reduced Singular Value Decomposition (SVD) of $X$.
The standard dot product between two vectors $x, y \in \mathbb{R}^d$ is defined as $\langle x, y \rangle := x^T y$. For matrices $X, Y \in \mathbb{R}^{m \times n}$, we define the inner product as $\langle X, Y \rangle := \mathrm{tr}(X^T Y)$.
%\textcolor{blue}{Add all norms: l1, l2, S1, S2(F), add dot product and ...} \textcolor{blue}{} \textcolor{blue}{$U_A, V_A$ come from the SVD of $A$: $A = U_A \Sigma_A V_A^T$}

We now provide several assumptions that are necessary for the analysis.
\begin{assumption}[Smoothness]\label{as:lip}
    The functions $f(x, \xi)$ are $L(\xi)$-smooth on the $\mathbb{R}^d$ with respect to the Euclidean norm $\norm{\cdot}$, i.e., for all $x, y \in \mathbb{R}^d$ it holds that
    $$
        \|\nabla f(x, \xi) - \nabla f(y, \xi)\|_2 \leq L(\xi) \|x-y\|_2.
    $$
    We also assume that exists constant $L^2 := \expect{L(\xi)^2}$.
\end{assumption}
\begin{assumption}[Bounded variance of the gradient] \label{as:sigma}
    The variance of the $\nabla f(x, \xi)$ is bounded with respect to the Euclidean norm, i.e., there exists $\sigma > 0$, such that for all $x \in \mathbb{R}^d$ it holds that
    $$
        \expect{\|\nabla f(x, \xi) - \nabla f(x)\|_2^2} \leq \sigma^2.
    $$
\end{assumption}
We assume access only to a zero-order oracle, which returns a noisy evaluation of the function $f(x, \xi)$. Therefore, we are limited to using this noisy value $\hat{f}(x, \xi)$ in the estimation of the ZO gradient \eqref{eq:zo_grad}. This noise may originate not only from inherent randomness (stochastic noise), but also from systematic effects (deterministic noise), such as computer rounding errors. Therefore, we make a common assumption about the function $\hat{f}(x, \xi)$ returned by the oracle \cite{dvurechensky2021accelerated, veprikov2024new}.
\begin{assumption}[Bounded oracle noise]\label{as:delta}
    The noise in the oracle is bounded with respect to the Euclidean norm, i.e., there exists $\Delta > 0$, such that for all $x \in \mathbb{R}^d$ it holds that
    $$
        \expect{\big|\hat{f}(x, \xi) - f(x, \xi)\big|^2} \leq \Delta^2 .
    $$
\end{assumption}

Assumptions \ref{as:lip} and \ref{as:sigma} are standard in the theoretical analysis of stochastic non-convex zero-order optimization problems \cite{reddy2023convergence, extreme_sparsity, liu2024sparse, wang2024simultaneous}. In contrast, Assumption \ref{as:delta} is frequently omitted in the existing literature, as it is commonly presumed that $\Delta = 0$, implying access to an ideal zero-order oracle. However, this assumption does not hold in practice, as numerical errors such as machine precision inevitably introduce a non-zero perturbation. Consequently, while $\Delta$ is typically small, it is never zero, which does not allow us to restore a true gradient along the direction $e$ in the estimation \eqref{eq:zo_grad} if we set $\tau \to 0$.

\subsection{Zero-Order Momentum SignSGD with JAGUAR Gradient Approximation} \label{subsec:jaguar}

In this section, we introduce zero-order SignSGD algorithm with JAGUAR gradient approximation \cite{veprikov2024new, nazykov2024stochastic} and momentum of the form:
    
\begin{algorithm}{Zero-Order Momentum SignSGD with JAGUAR (\texttt{JAGUAR SignSGD})}
   \label{algorithm:jaguar}
\begin{algorithmic}[1]
    \State {\bf Parameters:} stepsize (learning rate) $\gamma$, momentum $\beta$, smoothing parameter $\tau$, number of iterations $T$.
    \State {\bf Initialization:} choose  $x^{0} \in \mathbb{R}^d$ and $m^{-1} = \mathbf{0} \in \mathbb{R}^d$.
    \For{$t = 0, 1, 2, \dots, T$}
        \State Sample $i_t \sim \text{ Uniform}(\overline{1, d})$
        \State Set one-hot vector $e^t$ with $1$ in the $i_t$ coordinate: $e^t_{i_t} = 1$ and $e^t_{i \neq i_t} = 0$ for all $i \in \overline{1, d}$
        \State Sample stochastic variable $\xi^t \sim \mathcal{D}$
        \State Compute $\widetilde{\nabla}_{i_t} f(x^{t}, \xi^t) := \frac{\hat{f}(x^t + \tau e^t, \xi^t) - \hat{f}(x^t - \tau e^t, \xi^t)}{2 \tau} \in \mathbb{R}$
        \State Set $m^{t}_{i_t} = \beta m^{t-1}_{i_t} + (1 - \beta) \widetilde{\nabla}_{i_t} f(x^{t}, \xi^t)$ and $m^{t}_{i \neq i_t} = m^{t-1}_{i \neq i_t}$ for all $i \in \overline{1, d}$ \label{line:mt_jaguar_muon}
        \State Set $x^{t+1} = x^t - \gamma \cdot \text{sign}(m^{t})$
    \EndFor
    \State {\bf Return:} $x^{N(T)}$, where $N(T) \sim \text{ Uniform}(\overline{1, T})$.
\end{algorithmic}
\end{algorithm}

The gradient approximation employed in Algorithm \ref{algorithm:jaguar} deviates from that of the original JAGUAR method, as we introduce a momentum variable $\beta$. The estimator from the original work can be recovered by setting $\beta = 0$.

We now present a lemma characterizing the closeness between the momentum variable $m^t$ from line \ref{line:mt_jaguar} of Algorithm \ref{algorithm:jaguar} and the true gradient $\nabla f(x^t)$.

\begin{lemma}
    \label{lemma:mt_jaguar}
        Consider $m^t$ from line \ref{line:mt_jaguar} of Algorithm \ref{algorithm:jaguar}. Under Assumptions \ref{as:lip}, \ref{as:sigma}, \ref{as:delta} it holds that:
        \begin{align*}
            \expect{\norm{m^t\!-\!\nabla f(x^t)}_2^2}\!=\!
            \mathcal{O}\Bigg[ 
                &\frac{d^3  L^2 \gamma^2}{(1\!-\!\beta)^2}
                +
                (1\!-\!\beta)d \sigma^2
                +
                d L^2 \tau^2 
                +
                \frac{2 d \Delta^2}{\tau^2} 
                + 
                \left( \frac{1\!-\!\beta}{d} \right)^t \norm{\nabla f(x^0)}_2^2
            \Bigg] .
        \end{align*}
    \end{lemma}

\textbf{Discussion.} This lemma closely parallels Lemma 1 from \cite{veprikov2024new}, with the key distinction that our analysis incorporates the momentum parameter $\beta$, which was not present in \cite{veprikov2024new}. The introduction of momentum is essential for proving convergence of algorithms such as SignSGD (Algorithm \ref{algorithm:jaguar}) and Muon (see Algorithm \ref{algorithm:muon} in the next section) in the stochastic zero-order setting \cite{sun2023momentum}, as it enables more careful handling of variance $\sigma$ in the gradient estimates \eqref{eq:zo_grad}.
Another important difference from prior works is that result from Lemma \ref{lemma:mt_jaguar} does not involve the term $\|\nabla f(x^t)\|_2^2$, which typically appears in analyses where the zero-order gradient estimator \eqref{eq:zo_grad} is constructed using random uniform or Gaussian vectors $e$ \cite{cai2021zero, kozlak2021zero, gorbunov2022accelerated, qian2023zero}. With the presence of the term $\|\nabla f(x^t)\|_2^2$, it is not possible to achieve convergence guarantees for SignSGD (Algorithm \ref{algorithm:jaguar}) and Muon (Algorithm \ref{algorithm:muon}) even with momentum in the stochastic zero-order setting.
It is worth noting that a similar result can be obtained when using a full coordinate estimator \cite{lian2016comprehensive}. However, this approach requires $\mathcal{O}(d)$ calls to the zero-order oracle per iteration, which can be computationally expensive. In contrast, the JAGUAR method achieves the same result with only $\mathcal{O}(1)$ oracle calls and with the same number of parameters, offering significant improvements in efficiency. This makes our approach particularly attractive for large-scale optimization tasks, where reducing oracle complexity is critical.
In Appendix \ref{subsec:beta_ablation}, we provide an ablation study on $\beta$ and show that the \texttt{Jaguar SignSGD} method perform poorly for small $\beta$, while achieving robust high performance around $\beta \approx 0.9$.

With the help of Lemma \ref{lemma:mt_jaguar}, we provide convergence analysis of \texttt{JAGUAR SignSGD} (Algorithm \ref{algorithm:jaguar}).

\begin{theorem}
    \label{theorem:jaguar_sign}
    Consider Assumptions \ref{as:lip}, \ref{as:sigma} and \ref{as:delta}. Then \texttt{JAGUAR SignSGD} (Algorithm \ref{algorithm:jaguar}) has the following convergence rate:
        \begin{equation*}
        \begin{split}
            \expect{\norm{\nabla f\left(x^{N(T)} \right)}_1}
            =
            \mathcal{O} \Bigg[ 
                &\frac{\delta_0}{\gamma T}
                +
                \frac{d \norm{\nabla f(x^0)}_2}{T \sqrt{1 - \beta}}
                +
                \frac{d^2 L \gamma}{1-\beta}  
                +
                \sqrt{1-\beta}d \sigma
                +
                d L \tau
                +
                \frac{d \Delta}{\tau} 
            \Bigg],
        \end{split}
        \end{equation*}
        where we used a notation $\delta_0 := f(x^0) - f^*$. 
\end{theorem}

\begin{corollary}
\label{cor:jaguar_sign}
    Consider the conditions of Theorem \ref{theorem:jaguar_sign}. In order to achieve the $\varepsilon$-approximate solution (in terms of $\expect{\norm{\nabla f(x^{N(T)})}_1} \leq \varepsilon$), Algorithm \ref{algorithm:jaguar} needs $T$ iterations (ZO oracle calls), for:

    \textbf{Arbitrary tuning:}
    $
        \gamma = \gamma_0 \cdot T^{-3/4} d^{-1}, 
        ~ \beta = 1 - T^{-1/2},
        ~ \tau = \left( \Delta / L \right)^{1/2}
        \text{ and }
        \varepsilon \geq d \sqrt{\Delta L}:
    $
    \begin{equation*}
        T = \mathcal{O} \left[ \left( \frac{d \delta_0 / \gamma_0 + d \norm{\nabla f(x^0)}_2 + d L \gamma_0 + d \sigma }{\varepsilon} \right)^4\right] .
    \end{equation*}
    \textbf{Optimal tuning:}
    $
        \gamma = \sqrt{\frac{\delta_0 (1 - \beta)}{d^2 L T}}, 
        ~ \beta = 1 - \min\left\{1 ; \sqrt{\frac{L \delta_0}{T \sigma^2}}\right\},
        ~ \tau = \left( \Delta / L \right)^{1/2}
        \text{ and }
        \varepsilon \geq d \sqrt{\Delta L}:
    $
    \begin{equation*}
        T = \mathcal{O} \left[ \frac{\delta_0 L d^2}{\varepsilon^2} + \frac{\delta_0 L d^2}{\varepsilon^2} \cdot \left( \frac{d \sigma}{\varepsilon} \right)^2\right] .
    \end{equation*}
\end{corollary}

\textbf{Discussion.} The convergence rate established in Theorem \ref{theorem:jaguar_sign} is similar to what is known for first-order methods \cite{bernstein2018signsgd, jin2020stochastic, safaryan2021stochastic, kornilov2025sign}, however our bounds include an additional factor of $d$, which is typical for all coordinate-based methods \cite{nesterov2012efficiency, richtarik2016distributed}, not just zero-order ones. This dependence on the dimension arises because coordinate methods process one direction at a time, accumulating complexity proportional to $d$.
It is also important to note that without momentum ($\beta = 0$), the algorithm can only guarantee convergence to a neighbourhood of the optimum of size proportional to $\sigma$, as shown in previous works on zero-order SignSGD \cite{liu2019signsgd, kornilov2025sign}. Let us also point out that we cannot choose an arbitrary $\varepsilon$ in Corollary \ref{cor:jaguar_sign}, since there exists an irreducible \cite{dvurechensky2021accelerated, veprikov2024new} error $\Delta$ in the zero-order oracle (see Assumption \ref{as:delta}). However, since $\Delta$ is very small, we can still achieve an acceptable accuracy $\varepsilon$.
In our analysis, we use the $\ell_1$-norm of the gradient as the convergence criterion, while the standard in non-convex optimization is the $\ell_2$-norm (Euclidean) \cite{ghadimi2013stochastic,ghadimi2016accelerated}. By setting $\varepsilon_{\ell_1} = \sqrt{d} \cdot \varepsilon_{\ell_2}$, we can rescale our result of Corollary \ref{cor:jaguar_sign} for optimal tuning (one can easily do a similar transformation for arbitrary tuning) as
\begin{equation*}
    T_{\text{Euclidean}} = \mathcal{O} \left[ \frac{\delta_0 L d}{\varepsilon^2} + \frac{\delta_0 L d}{\varepsilon^2} \cdot \left( \frac{\sqrt{d} \sigma}{\varepsilon} \right)^2\right] .
\end{equation*}
This substitution allows us to obtain improved results in terms of the dependence on $d$.

\subsection{Zero-Order Muon with JAGUAR Gradient Approximation}

In this section, we address the matrix optimization setting, where the optimization variables $X_t$ are elements of the matrix space $\mathbb{R}^{m \times n}$, rather than the standard vector space $\mathbb{R}^d$. Such a formulation allows for a more direct representation of model parameters, helping to better capture their underlying structure \cite{bernstein2024old, pethick2025training}. For the first time in the literature, we introduce a zero-order version of the Muon \cite{muon_base} algorithm (Algorithm \ref{algorithm:muon}), broadening the applicability to matrix-structured optimization tasks where only function evaluations are available.

\begin{algorithm}{Zero-Order Muon with JAGUAR (\texttt{JAGUAR Muon})}
   \label{algorithm:muon}
\begin{algorithmic}[1]
    \State {\bf Parameters:} stepsize (learning rate) $\gamma$, momentum $\beta$,  smoothing parameter $\tau$, number of Newton-Schulz steps $\text{ns\_steps}$, number of iterations $T$.
    \State {\bf Initialization:} choose  $X^{0} \in \mathbb{R}^{m \times n}$ and $M^{-1} = \mathbf{0} \in \mathbb{R}^{m \times n}$.
    \For{$t = 0, 1, 2, \dots, T$}
        \State Sample $i_t \sim \text{ Uniform}(\overline{1, m})$ and $j_t \sim \text{ Uniform}(\overline{1, n})$
        \State Set one-hot matrix $E^t$ with $1$ in the $(i_t, j_t)$ coordinate
        \State Sample stochastic variable $\xi^t \sim \mathcal{D}$
        \State Compute $\widetilde{\nabla}_{i_t,j_t} f(X^{t}, \xi^t) := \frac{\hat{f}(X^t + \tau E^t, \xi^t) - \hat{f}(X^t - \tau E^t, \xi^t)}{2 \tau} \in \mathbb{R}$
        \State Set $M^{t}_{i_t, j_t} = \beta M^{t-1}_{i_t, j_t} + (1 - \beta) \widetilde{\nabla}_{i_t,j_t} f(x^{t}, \xi^t)$ and $M^{t}_{i \neq i_t, j \neq j_t} = M^{t-1}_{i \neq i_t, j \neq j_t}$ \label{line:mt_jaguar}
        \State Set $X^{t+1} = X^t - \gamma \cdot \texttt{Newton\_Schulz}(M^{t}, K=\text{ns\_steps})$
    \EndFor
    \State {\bf Return:} $X^{N(T)}$, where $N(T) \sim \text{ Uniform}(\overline{1, T})$.
\end{algorithmic}
\begin{algorithmic}[1]
    \State \textbf{Subroutine} $\texttt{Newton\_Schulz}(A \in \mathbb{R}^{m \times n}, K = 5)$ \cite{bernstein2024old}:
    \State \quad Set $A^0 = A / \| A \|_F$
    \State \quad \textbf{for} {$k = 0, 1, 2, \dots, K$} \textbf{do}
    \State \quad \quad $A^{k+1} = 3/2 \cdot A^k - 1/2 \cdot A^k (A^k)^T A^k$
    \State \quad \textbf{end for}
    \State \quad {\bf Return:} $A^K \approx U_A \cdot V_A^T$.
    %\Comment{$U_A, V_A$ come from the SVD of $A$: $A = U_A \Sigma_A V_A^T$}
\end{algorithmic}
\end{algorithm}

Algorithm \ref{algorithm:muon} is similar to the first-order Muon algorithm \cite{muon_base}, the only difference is that we use zero-order gradient approximation JAGUAR \cite{veprikov2024new} in line \ref{line:mt_jaguar_muon}. 

Let us note that when extending to matrix-valued parameters, it is necessary to slightly modify Assumptions \ref{as:lip} and \ref{as:sigma}: all occurrences of the $\ell_2$ norm $\| \cdot \|_2$ should be replaced with the Frobenius norm $\| \cdot \|_F$. This modification is justified, as the following property holds for all matrices $A \in \mathbb{R}^{m \times n}$: $\| A \|_F = \|\overline{\text{vec}}(A)\|_2$.
We now provide the convergence analysis of \texttt{JAGUAR Muon} (Algorithm \ref{algorithm:muon}). 

\begin{theorem}
    \label{theorem:jaguar_muon}
    Consider Assumptions \ref{as:lip}, \ref{as:sigma} (with Frobenius norm) and \ref{as:delta}. Then \texttt{JAGUAR Muon} (Algorithm \ref{algorithm:muon}) has the following convergence rate:
        \begin{equation*}
        \begin{split}
            \expect{\norm{\nabla f\left(X^{N(T)} \right)}_{\mathcal{S}_1}}
            =
            \mathcal{O} \Bigg[ 
                &\frac{\delta_0}{\gamma T}
                +
                \frac{m^{1/2}n \norm{\nabla f(X^0)}_2}{T \sqrt{1 - \beta}}
                +
                \frac{m^{3/2}n^2 \gamma}{1-\beta}  
                +
                \sqrt{1-\beta} m^{1/2}n \sigma
                \\&+
                m^{1/2}n L \tau
                +
                \frac{m^{1/2}n \Delta}{\tau} 
            \Bigg],
        \end{split}
        \end{equation*}
        where we used a notation $\delta_0 := f(x^0) - f^*$. We also assume that $n \leq m$.
\end{theorem}

\begin{corollary}
\label{cor:muon}
    Consider the conditions of Theorem \ref{theorem:jaguar_muon}. In order to achieve the $\varepsilon$-approximate solution (in terms of $\mathbb{E}[\| \nabla f(X^{N(T)})\|_{\mathcal{S}_1}] \leq \varepsilon$), Algorithm \ref{algorithm:muon} needs $T$ iterations (ZO calls), for:

    \textbf{Arbitrary tuning:}
    $
        \gamma = \gamma_0 \cdot T^{-3/4} (mn)^{-1}, 
        \beta = 1 - T^{-1/2},
        \tau = \left( \Delta / L \right)^{1/2},
        \varepsilon \geq m^{1/2} n \sqrt{\Delta L}:
    $
    \begin{equation*}
        T = \mathcal{O} \left[ \left( \frac{mn \delta_0 / \gamma_0 + m^{1/2} n \norm{\nabla f(X^0)}_2 + m^{1/2}n L \gamma_0 + m^{1/2}n \sigma }{\varepsilon} \right)^4\right] .
    \end{equation*}
    \textbf{Optimal tuning:}
    $
        \gamma = \sqrt{\frac{\delta_0 (1 - \beta)}{m^{3/2}n^2 L T}}, 
        \beta = 1\!-\!\min\left\{1 ; \sqrt{\frac{L \delta_0}{T \sigma^2}}\right\},
        \tau = \left( \Delta / L \right)^{1/2},
        \varepsilon \geq m^{1/2} n \sqrt{\Delta L}:
    $
    \begin{equation*}
        T = \mathcal{O} \left[ \frac{\delta_0 L m^{3/2}n^2}{\varepsilon^2} + \frac{\delta_0 L m^{3/2}n^2}{\varepsilon^2} \cdot \left( \frac{m^{3/2}n^2 \sigma}{\varepsilon} \right)^2\right] .
    \end{equation*}
\end{corollary}

\textbf{Discussion.}
The convergence rate established in Theorem \ref{theorem:jaguar_muon} is consistent with the first-order case \cite{muon_convergence, kovalev2025understanding}. However, there remain zero-order terms depending on $\tau$ and $\Delta$, as for Algorithm \ref{algorithm:jaguar} (see Theorem \ref{theorem:jaguar_sign} and Discussion part after it). From a proof perspective, Theorems \ref{theorem:jaguar_sign} and \ref{theorem:jaguar_muon} are very similar, since the orthogonalization operation (\texttt{Newton\_Schulz}) in Algorithm \ref{algorithm:muon} can be interpreted as taking the sign of the gradient matrix eigenvalues. Accordingly, both the form and the convergence rate criterion are analogous (the $\ell_1$ norm for Algorithm \ref{algorithm:jaguar} and the $\mathcal{S}_1$ norm for Algorithm \ref{algorithm:muon}). Nevertheless, the convergence rates of the two algorithms differ slightly. We examine the two boundary cases in the following remark.

\begin{remark} For optimal tuning from Corollary \ref{cor:muon} we can specify the number of iterations of Algorithm \ref{algorithm:muon} to achieve the $\varepsilon$-approximate solution in terms of the total number of parameters $d = m \cdot n$ in the two boundary cases: 
\begin{enumerate}[leftmargin=7pt]
    \item[$\bullet$] If $n \ll m \approx d$:
    $$
        \quad T_{n \ll m \approx d} = \mathcal{O} \left[ \frac{\delta_0 L d^{3/2}}{\varepsilon^2} + \frac{\delta_0 L d^{3/2}}{\varepsilon^2} \cdot \left( \frac{d^{3/2} \sigma}{\varepsilon} \right)^2\right] .
    $$
    \item[$\bullet$] If $n \approx m \approx \sqrt{d}$:
    $$ 
        \quad T_{n \approx m \approx \sqrt{d}} = \mathcal{O} \left[ \frac{\delta_0 L d^{7/4}}{\varepsilon^2} + \frac{\delta_0 L d^{7/4}}{\varepsilon^2} \cdot \left( \frac{d^{7/4}\sigma}{\varepsilon} \right)^2\right] .
    $$
\end{enumerate}
\end{remark}
Accordingly, comparing these convergence rates with that obtained in Corollary \ref{cor:muon}, we observe an improvement by factors of $d^{1/2}$ and $d^{1/4}$, respectively.

\textcolor{blue}{}

\section{Experiments} \label{sec:exp}

In this section, we present a comprehensive empirical evaluation to validate the theoretical contributions of our proposed ZO optimization methods for fine-tuning large language models. Our study aims to assess both the accuracy and memory efficiency of these methods, comparing them against established ZO and FO baselines. We build upon the experimental framework proposed in \cite{zo_bench}, extending it to incorporate our novel algorithms: \texttt{JAGUAR SignSGD} (Algorithm \ref{algorithm:jaguar}) and \texttt{JAGUAR Muon} (Algorithm \ref{algorithm:muon}). The primary objective is to achieve competitive test accuracy on downstream tasks while maintaining memory efficiency comparable to the baseline methods. Additionally, we introduce \texttt{ZO-Muon} (Algorithm \ref{algorithm:zo_muon_base} in Appendix \ref{subsec:mezo_generation}), a direct zero-order adaptation of Muon \cite{muon_base}, utilizing the standard Gaussian zero-order gradient estimation \eqref{eq:zo_grad}. % We also explore various matrix sampling techniques to enhance this approach, as detailed in Appendix \ref{subsec:mezo_generation}. 

\subsection{Experimental Setup}

\textbf{Fine-Tuning Task and Schemes.} Fine-tuning LLMs is a pivotal process in adapting pre-trained models to downstream tasks, enabling high performance with limited task-specific data. 
%This process typically involves adjusting model parameters to minimize a task-specific loss function, a task that becomes computationally intensive as model sizes scale to billions of parameters. 
To explore the efficacy of our ZO methods, we focus on the SST2 dataset \cite{socher2013recursive}, a widely-used benchmark for binary sentiment classification \cite{zo_bench,lozo,malladi2023fine}. Additionally, we measure performance of Llama2-7B \cite{touvron2023llama} and OPT-13B \cite{zhang2022opt} on WinoGrande \cite{sakaguchi2021winogrande} and COPA \cite{roemmele2011choice} datasets (see Appendix \ref{appendix:exp_setup} for details). We consider two fine-tuning schemes:
\begin{itemize}[leftmargin=7pt]
    \item \textbf{Full Fine-Tuning (FT):} Updates all parameters of the pre-trained model, offering maximum flexibility at the cost of higher computational resources.
    \item \textbf{Low-Rank Adaptation (LoRA):} Introduces a small set of trainable parameters while keeping the original model parameters frozen, enhancing memory efficiency \cite{hu2021lora}.
\end{itemize}

\textbf{Models.} We conduct experiments using four prominent LLMs: OPT-1.3B~\cite{zhang2022opt}, a 1.3 billion parameter model from the OPT family; RoBERTa-Large~\cite{liu2019roberta}, a 355 million parameter model known for its robust performance in natural language processing tasks; Llama 2~\cite{touvron2023llama} and OPT-13B~\cite{zhang2022opt}, state-of-the-art open-source models widely used for research and applications. These models represent a range of sizes and architectures, allowing us to assess the scalability and generality of our methods.

\textbf{Methods.} We evaluate the following ZO optimization methods proposed in this work:
\begin{itemize}[leftmargin=7pt]
    \item \textbf{\texttt{JAGUAR SignSGD}:} Combines the JAGUAR gradient approximation \cite{veprikov2024new} with SignSGD and momentum for efficient updates (Algorithm \ref{algorithm:jaguar}).
    \item \textbf{\texttt{JAGUAR Muon}:} Integrates JAGUAR with the Muon optimizer, incorporating momentum and orthogonalization (Algorithm \ref{algorithm:muon}).
    \item \textbf{\texttt{ZO-Muon}:} A novel ZO adaptation of the Muon optimizer, leveraging matrix-based optimization principles (Algorithm \ref{algorithm:zo_muon_base} in Appendix \ref{subsec:mezo_generation}).
\end{itemize}

\textbf{Comparison procedure.}
For comparison, we include baseline methods from \cite{zo_bench}: ZO-SGD \cite{ghadimi2013stochastic}, Acc-ZOM \cite{huang2022accelerated}, ZO-SGD-Cons \cite{kim2025conserv}, ZO-SignSGD \cite{liu2019signsgd}, ZO-AdaMM \cite{chen2019zo}, Forward-Grad \cite{baydin2022gradientsbackpropagation}, and the FO method FO-SGD \cite{amari1993backpropagation}. 
The results for which are given in the benchmark paper. Additionally, we compare our methods with LeZO \cite{wang2024simultaneous}, which employs a comparable layer-wise selection mechanism similar to \texttt{JAGUAR SignSGD} coordinate-wise updates. We perform experiments for our methods in accordance with similar experiments from \cite{zo_bench}. For details of our hyperparameter selection and model training procedure, see Appendix \ref{appendix:exp_setup}.

\begin{table}[!pt]
    \centering
    \caption{Test accuracy on SST2 for OPT-1.3B and RoBERTa-Large with FT and LoRA. Best performance among ZO methods is in \textbf{bold}. \textcolor{blue}{Blue} indicates outperformance of all baseline ZO methods, \textcolor{purple}{red} indicates matching or exceeding FO-SGD.}
    \resizebox{\linewidth}{!}{
    \begin{tabular}{l|cc|cc}
    \toprule
    Method & \multicolumn{2}{c}{OPT-1.3B} & \multicolumn{2}{c}{RoBERTa-Large} \\
    \cmidrule(lr){2-3} \cmidrule(lr){4-5}
    & FT & LoRA & FT & LoRA \\
    \midrule
    FO-SGD & 91.1 & 93.6 & 91.4 & 91.2 \\
    Forward-Grad & 90.3 & 90.3 & 90.1 & 89.7 \\
    ZO-SGD & 90.8 & 90.1 & 89.4 & 90.8 \\
    Acc-ZOM & 85.2 & 91.3 & 89.6 & 90.9 \\
    ZO-SGD-Cons & 88.3 & 90.5 & 89.6 & 91.6 \\
    ZO-SignSGD & 87.2 & 91.5 & 52.5 & 90.2 \\
    ZO-AdaMM & 84.4 & 92.3 & 89.8 & 89.5 \\
    LeZO & 85.1 & 92.3 & 90.4 & 91.8\\
    \midrule
    \cellcolor{bgcolor2} \texttt{JAGUAR SignSGD} & \cellcolor{bgcolor2}{\textbf{\textcolor{purple}{94.0 $\pm$ 0.1}}} & \cellcolor{bgcolor2}{\textcolor{blue}{92.5 $\pm$ 0.5}} & \cellcolor{bgcolor2}{\textbf{\textcolor{purple}{92.2 $\pm$ 0.2}}} & \cellcolor{bgcolor2}{\textbf{\textcolor{purple}{92.2 $\pm$ 0.4}}} \\
    
    \cellcolor{bgcolor2} \texttt{JAGUAR Muon} & \cellcolor{bgcolor2}{84.0 $\pm$ 0.1} & \cellcolor{bgcolor2}{\textbf{\textcolor{purple}{94.0 $\pm$ 0.1}}} & \cellcolor{bgcolor2}85.0 $\pm$ 0.1 & \cellcolor{bgcolor2}{\textbf{\textcolor{purple}{92.2 $\pm$ 0.2}}} \\
    
    \cellcolor{bgcolor2} \texttt{ZO-Muon} & \cellcolor{bgcolor2}{86.5 $\pm$ 0.1} & \cellcolor{bgcolor2}{\textcolor{purple}{93.5 $\pm$ 0.1}} & \cellcolor{bgcolor2}{72.0 $\pm$ 0.1} & \cellcolor{bgcolor2}{86.0 $\pm$ 0.2}  \\
    \bottomrule
    \end{tabular}
    }
    \label{tab:results}
\end{table}

\begin{table}[bp!]
    \centering
    \caption{Test accuracy on COPA and WinoGrande for OPT-13B and Llama2-7B with LoRA. Best performance among ZO methods is in \textbf{bold}. \textcolor{blue}{Blue} indicates outperformance of all baseline ZO methods, \textcolor{purple}{red} indicates matching or exceeding FO-SGD.}
    \resizebox{\linewidth}{!}{
    \begin{tabular}{l|cc|cc}
    \toprule
    Method & \multicolumn{2}{c}{OPT-13B} & \multicolumn{2}{c}{LLaMA2-7B} \\
    \cmidrule(lr){2-3} \cmidrule(lr){4-5}
    & COPA & WinoGrande & COPA & WinoGrande \\
    \midrule
    FO-SGD & 88 & 66.9 & 85 & 66.9 \\
    Forward-Grad & 89 & 62.9 & 82 & 64.3 \\
    ZO-SGD & 87 & 62.6 & 86 & 64.3 \\
    ZO-SGD-Cons & 88 & 63.3 & 85 & 64.6 \\
    % \textcolor{blue}{LOZO} & \textcolor{blue}{89.0} & - & - & - \\
    % \textcolor{blue}{HiZOO} & \textcolor{blue}{88.0} & - & - & - &\\
    \midrule
    \cellcolor{bgcolor2} \texttt{JAGUAR SignSGD} & \cellcolor{bgcolor2}\textbf{\textcolor{purple}{89 $\pm$ 0.3}} & \cellcolor{bgcolor2}\textbf{\textcolor{blue}{63.7$\pm$ 0.1}} & \cellcolor{bgcolor2}\textbf{\textcolor{purple}{88 $\pm$ 0.2}} &  \cellcolor{bgcolor2}\textbf{\textcolor{blue}{64.9$\pm$ 0.1}} \\
    
    \cellcolor{bgcolor2} \texttt{JAGUAR Muon} & \cellcolor{bgcolor2}{87 $\pm$ 0.2} & \cellcolor{bgcolor2}62.3 $\pm$ 0.2 & \cellcolor{bgcolor2}\textbf{\textcolor{purple}{88 $\pm$ 0.1}}& \cellcolor{bgcolor2}62.8 $\pm$ 0.2\\

    \cellcolor{bgcolor2} \texttt{ZO-Muon} & \cellcolor{bgcolor2} 87 $\pm$ 0.2 & \cellcolor{bgcolor2}61.9 $\pm$ 0.3 & \cellcolor{bgcolor2} 85 $\pm$ 0.2 & \cellcolor{bgcolor2} 61.6 $\pm$ 0.2\\
    
    \bottomrule
    \end{tabular}
    }
    \label{tab:llama}
\end{table}

\subsection{Results}
\textbf{OPT-1.3B and RoBERTa-Large models.} Table \ref{tab:results} presents the test accuracy results for SST2 across both OPT-1.3B and RoBERTa-Large models and fine-tuning schemes. Our proposed methods demonstrate strong performance, often outperforming baseline ZO methods. Based on the results presented in Table \ref{tab:results}, proposed methods (Algorithms \ref{algorithm:jaguar} and \ref{algorithm:muon}) that leverage the JAGUAR approximation of gradient outperform comparable approaches utilizing standard random vector sampling $e$ in equation \eqref{eq:zo_grad} or vanilla momentum techniques originally designed for FO algorithms. However, \texttt{ZO-Muon} and \texttt{JAGUAR Muon} show reduced FT performance, potentially due to the presence of non-matrix parameters in the full FT process.

\textbf{OPT-13B and Llama2-7B models.} To justify the reliability of the proposed methods, we conduct additional experiments with large-size models: OPT-13B \cite{zhang2022opt} and Llama2-7B \cite{touvron2023llama} on WinoGrande \cite{sakaguchi2021winogrande} and COPA \cite{roemmele2011choice} tasks. Within this series of evaluations, we implement a learning schedulers —cosine for Llama2-7B and polynomial decay for OPT-13B. We repeat the evaluation results from \cite{zo_bench} as baselines in Table \ref{tab:llama}. However, in the mentioned work, the authors do not report memory efficiency, which is a sufficient indicator in parameter-efficient fine-tuning competition. The ZO-AdaMM method was not considered in our experiments due to its prohibitively high memory requirements. 
%This series of experiments on large-scale models further reinforces that the proposed methods consistently outperform the baselines. 

\textbf{Discussion.} The results from Tables \ref{tab:results} and \ref{tab:llama} demonstrate that \texttt{JAGUAR SignSGD} and \texttt{JAGUAR Muon} achieve superior performance, demonstrating the effectiveness and robustness compared to existing baselines. Our methods excel in real-world applications, particularly where memory limits hinder traditional FO techniques. The results demonstrate the effectiveness and scalability of our approaches, confirming their advantages in challenging, high-capacity settings.

\subsection{Memory Efficiency}
% \begin{wrapfigure}[8]{r}{6cm}
% \vspace{-13mm}
% \begin{minipage}{6cm}
Tables \ref{tab:memory} and \ref{tab:llama_memory} compares GPU allocated memory for OPT-1.3B, Llama-7B and OPT-13B highlighting the efficiency of our methods. Results of this experiment demonstrate that our approaches effectively balance accuracy gains with memory efficiency.
\begin{table}[h!]
    \centering
    \caption{GPU allocated memory (GB) for OPT-1.3B (half-precision, F16) on SST2 with FT and LoRA.}
    \resizebox{0.7\linewidth}{!}{
    \begin{tabular}{lcc}
    \toprule
    Method & FT Memory & LoRA Memory \\
    \midrule
    FO-SGD & 12.246 & 5.855 \\
    ZO-SGD & 4.171 & 4.125 \\
    ZO-AdaMM & 13.046 & 6.132 \\
    \midrule
    \cellcolor{bgcolor2}\texttt{JAGUAR SignSGD} & \cellcolor{bgcolor2}4.172 & \cellcolor{bgcolor2}4.128 \\
    \cellcolor{bgcolor2}\texttt{JAGUAR Muon} & \cellcolor{bgcolor2}4.179 & \cellcolor{bgcolor2}4.132 \\
    \cellcolor{bgcolor2}\texttt{ZO-Muon} & \cellcolor{bgcolor2}4.177 & \cellcolor{bgcolor2}4.130 \\
    \bottomrule
    \end{tabular}
    }
    \label{tab:memory}
\end{table}
% \begin{table}[h!]
%     \centering
%     \caption{GPU allocated memory (GB) for OPT-13B and LLaMA2-7B (half-precision, F16) on WinoGrande and COPA with LoRA}
%     \resizebox{0.39\linewidth}{!}{
%     \begin{tabular}{lcc}
%     \toprule
%     Model & Llama-7B & OPT-13B \\
%     \midrule
%     \multicolumn{3}{c}{\textbf{COPA}} \\
%     \midrule
%     ZO-SGD & 13.219 & 24.710 \\
%     ZO-AdaMM & 27.971 & 38.612 \\
%     \cellcolor{bgcolor2}\texttt{JAGUAR SignSGD} & 1\cellcolor{bgcolor2}3.219 & \cellcolor{bgcolor2}24.712 \\
%     \cellcolor{bgcolor2}\texttt{JAGUAR Muon} & \cellcolor{bgcolor2}16.032 & \cellcolor{bgcolor2}25.880 \\
%     \cellcolor{bgcolor2}\texttt{ZO-Muon} & \cellcolor{bgcolor2}15.021 & \cellcolor{bgcolor2}25.740  \\
%     \midrule
%     \multicolumn{3}{c}{\textbf{WinoGrande}} \\
%     \midrule
%     ZO-SGD & 14.670 & 26.407 \\
%     ZO-AdaMM & 29.440 & 39.872 \\
%     \cellcolor{bgcolor2}\texttt{JAGUAR SignSGD} &  \cellcolor{bgcolor2}14.672  & \cellcolor{bgcolor2}26.408 \\
%     \cellcolor{bgcolor2}\texttt{JAGUAR Muon} & \cellcolor{bgcolor2}17.992 & \cellcolor{bgcolor2}27.440 \\
%     \cellcolor{bgcolor2}\texttt{ZO-Muon} & \cellcolor{bgcolor2}16.992 & \cellcolor{bgcolor2}27.416 \\
%     \bottomrule
%     \end{tabular}
%     }
%     \label{tab:llama_memory}
% \end{table}
\begin{table}[h!]
    \centering
    \caption{GPU allocated memory (GB) for OPT-13B and LLaMA2-7B on WinoGrande and COPA with LoRA}
    \resizebox{0.8\linewidth}{!}{
    \begin{tabular}{l|cc|cc}
    \toprule
    Method & \multicolumn{2}{c}{OPT-13B} & \multicolumn{2}{c}{LLaMA2-7B} \\
    \cmidrule(lr){2-3} \cmidrule(lr){4-5}
    & COPA & WinoGrande & COPA & WinoGrande \\
    \midrule
    ZO-SGD & 24.710 & 26.407 & 13.219 & 14.670 \\
    ZO-Adam & 38.612 & 39.872 & 27.971 & 29.440 \\
    \midrule
    \cellcolor{bgcolor2} \texttt{JAGUAR SignSGD} &\cellcolor{bgcolor2}24.712  & \cellcolor{bgcolor2}26.408 & \cellcolor{bgcolor2}13.219  & \cellcolor{bgcolor2}14.672 \\
    
    \cellcolor{bgcolor2} \texttt{JAGUAR Muon} & \cellcolor{bgcolor2}25.880 & \cellcolor{bgcolor2}27.440 & \cellcolor{bgcolor2}16.032 & \cellcolor{bgcolor2}17.992\\

    \cellcolor{bgcolor2} \texttt{ZO-Muon} & \cellcolor{bgcolor2}25.740 & \cellcolor{bgcolor2}27.416 & \cellcolor{bgcolor2}15.021 & \cellcolor{bgcolor2}16.992\\
    \bottomrule
    \end{tabular}
    }
    \label{tab:llama_memory}
\end{table}
% \end{minipage}
% \end{wrapfigure}
%Our methods consume approximately 4.17-4.18 GB in FT and 4.12-4.13 GB in LoRA, closely aligning with ZO-SGD, while FO-SGD requires 12.2 GB and 5.9 GB, respectively. 
%See other experiments with large-size models and technical details in Appendix \ref{appendix:exp_setup}.
% \section{Discussion}\label{sec:disc}
% TODO    \cellcolor{bgcolor2}\texttt{ZO-Muon} & \cellcolor{bgcolor2}16.992 & \cellcolor{bgcolor2}27.416 \\

% \section{Conclusion}\label{sec:concl}

% TODO

% \bibliographystyle{plain}
% \bibliography{references}

\end{mainpart}

\begin{appendixpart}
\section{Ablation study} \label{subsec:beta_ablation}

% We present an ablation study on the effect of the $\beta$ parameter from Algorithm \ref{algorithm:jaguar} on learning efficiency. Figure \ref{fig:beta_ablation} reports the accuracy of the \texttt{JAGUAR SignSGD} method on the SST-2 dataset with the RoBERTa-large model across different values of $\beta$.
% The method demonstrates substantially lower accuracy for small $\beta$, while attaining robust and consistently high performance around $\beta \approx 0.9$.

We present an ablation study on the effect of the $\beta$ parameter from Algorithm \ref{algorithm:jaguar} on learning efficiency. Figure \ref{fig:beta_ablation} reports the accuracy of the \texttt{JAGUAR SignSGD} method on the SST-2 dataset with the RoBERTa-large model across different values of $\beta$. The results indicate that the choice of $\beta$ has a significant impact on model performance. Specifically, small values of $\beta$ lead to substantially lower accuracy, suggesting that insufficient momentum or smoothing in the update steps can hinder effective learning. As $\beta$ increases, the method benefits from more stable gradient aggregation, resulting in improved convergence behavior. Notably, around $\beta \approx 0.9$, the method achieves robust and consistently high performance, indicating that this range provides an optimal balance between responsiveness to new gradient information and stability in updates. This highlights the importance of tuning $\beta$ carefully to maximize the learning efficiency and predictive performance of \texttt{JAGUAR SignSGD}.

\begin{figure}[H]
    \centering
    \includegraphics[width=\linewidth]{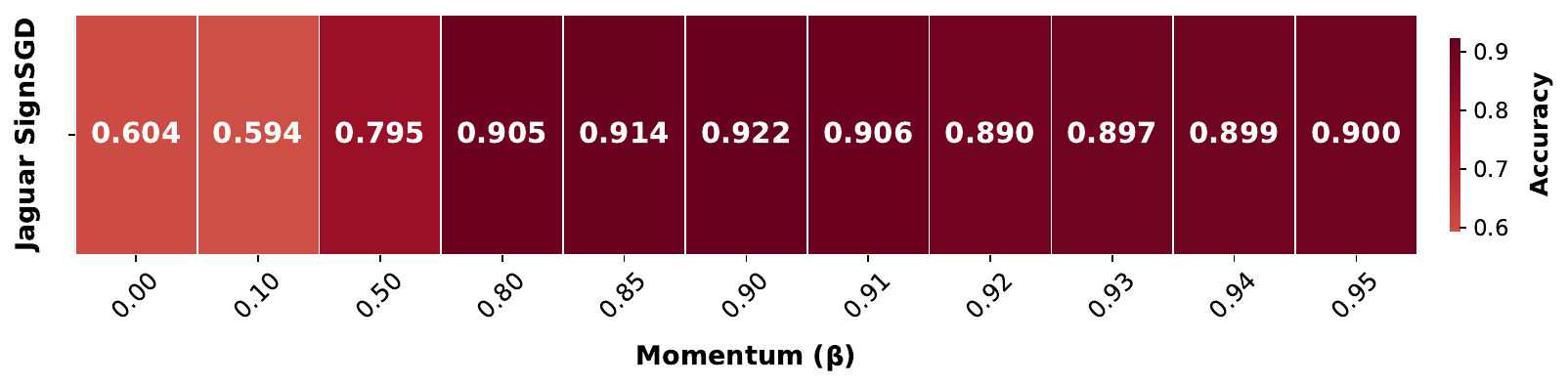}
    \caption{Test accuracy of \texttt{JAGUAR SignSGD} on SST-2 for RoBERTa-large with LoRA for different values of $\beta$.}
    \label{fig:beta_ablation}
\end{figure}

\section{Classical ZO Muon} \label{subsec:mezo_generation}
%\subsection{Zero-Order Muon}\label{subsec:zo_muon}
Using gradient estimate in the form \eqref{eq:zo_grad}, we adapt the Muon algorithm \cite{muon_base} into zero-order form:

\begin{algorithm}{Zero-Order Muon (\texttt{ZO-Muon})}
    \label{algorithm:zo_muon_base}
    \begin{algorithmic}[1]
        \State \textbf{Parameters:} stepsize (learning rate) $\gamma$, gradient approximation parameter $\tau$, number of iterations $T$.
        \State \textbf{Initialization:} choose $X_0 \in \mathbb{R}^{m \times n}$
        \For{$t = 0, 1, 2, \dots, T$}
            \State Sample $E^t \in \mathbb{R}^{m \times n}$ from $\mathcal{N}(0,1)$
            \State Compute $G^t = \frac{\hat{f}(X^t + \tau E^t) - \hat{f}(X^t - \tau E^t)}{2\tau}E^t$
            \State Set $X^{t+1} = X^t - \gamma \cdot \texttt{Newton\_Schulz}(G^t)$
        \EndFor
    \end{algorithmic}
\end{algorithm}

% In Section \ref{sec:mezo_muon}, we introduce a matrix-efficient variant of Algorithm \ref{algorithm:muon_sampling_base}, in which the semi-orthogonal matrix $E^t$ is sampled directly.

\section{Fine-Tuning Setup} \label{appendix:exp_setup}
%\textcolor{blue}{TELL ABOUT SCHEDULING AND THE RESULTS CLEARLY}

% TODO: explain what "Forward-Grad" is, add iterations for all of these configurations, and hyperparameters. Also mention that models are trained in half-precision (fp16), scheduler 

\subsection{Evaluation Procedure}

\textbf{Schedulers.} We conduct experiments with different scheduling types. Results for \texttt{Jaguar Muon} (Algorithm \ref{algorithm:muon}) and \texttt{Muon} (Algorithm \ref{algorithm:zo_muon_base}) from Tables \ref{tab:results} (only for FT) and \ref{tab:llama} are obtained using polynomial scheduling technique. The rest of the experiments are conducted without scheduling.

\textbf{Hyperparameter Tuning.} To ensure optimal performance, we conduct a grid search over key hyperparameters for each method:
\begin{itemize}[leftmargin=7pt]
    \item Momentum parameter: $\beta \in \{10^{-3}, 10^{-2}, 10^{-1}, 8\cdot 10^{-1}\}$,
    \item Learning rate: $\gamma \in [10^{-6}, 10^{-1}]$,
    \item Smoothing parameter: $\tau \in \{10^{-1}, 10^{-2}, 10^{-3}\}$.
\end{itemize}
Additional fixed parameters include an epsilon of $10^{-3}$ for numerical stability. The best-performing hyperparameters for each algorithm are detailed on our github \url{https://github.com/brain-lab-research/zero-order-optimization}.

\textbf{Evaluation Metrics.} We assess performance using:
\begin{itemize}[leftmargin=7pt]
    \item \textbf{Test Accuracy:} Measured as the percentage of correct predictions on the test set, reflecting model effectiveness.
    \item \textbf{GPU allocated memory:} Quantified in gigabytes (GB) during training, indicating memory efficiency.
\end{itemize}

\textbf{Implementation Details.} We conduct experiments with three independent runs per configuration, each with a randomly selected seed fixed at the start to ensure reproducibility. We report the mean and standard deviation of test accuracy. Following \cite{malladi2023fine}, we employ half-precision (F16) training for ZO methods and mixed-precision (FP16) training for FO methods to optimize memory usage. We use LoRA \cite{hu2021lora} fine-tuning strategy with $r=16$. We perform training on a single NVIDIA A100 GPU and a single NVIDIA H100 GPU, with memory profiling by standard PyTorch utilities.

%We report the GPU memory allocation for our methods and baselines in Table \ref{tab:llama_memory}.

\subsection{Experimental Methodology}

Our experimental procedure designed to rigorously evaluate the proposed methods under controlled conditions. We consider different datasets (SST2, COPA, WinoGrande), models (OPT-1.3B, RoBERTa-Large, Llama2 7B, OPT-13B), fine-tuning schemes (FT, LoRA), and ZO and FO optimization methods (see Tables \ref{tab:results} and \ref{tab:llama}). We execute the following steps:
\begin{enumerate}[leftmargin=15pt]
    \item \textbf{Initialization:} Load the pre-train model and initialize trainable parameters (all for FT, LoRA-specific for LoRA).
    \item \textbf{Hyperparameter Selection:} Perform a preliminary parameter search to identify the best hyperparameters per method, iterating over the specified ranges and selecting based on validation accuracy.
    \item \textbf{Evaluation:} Compute test accuracy on the dataset test set after each run, averaging results across three runs with different seeds.
    \item \textbf{Memory Profiling:} Record GPU allocated memory during training, ensuring consistency by maintaining identical hardware settings.
\end{enumerate}

This methodology ensures a fair comparison across methods, capturing both performance and resource utilization comprehensively.

\newpage
\section{Proofs for ZO Momentum SignSGD with JAGUAR (Algorithm \ref{algorithm:jaguar})}

    \subsection{Proof of Lemma \ref{lemma:mt_jaguar}}
    \begin{proof}
        We start with applying one step recursion to the momentum form the Algorithm \ref{algorithm:jaguar}:
        \begin{align}
            \expect{\norm{m^{t} - \nabla f(x^t)}_2^2} 
            &=
            \expectBig{\normBig{
                m^{t-1} 
                - (1-\beta) \dotprod{m^{t-1}}{e^t}e^t \nonumber
                \\&\qquad\quad+ (1-\beta) \widetilde{\nabla}_{i_t} f(x^{t}, \xi^t)
                - \nabla f(x^t)
            }_2^2}
            \nonumber
            \\&=
            \expectBig{\normBig{
                \left\{I - (1-\beta)e^t (e^t)^T \right\} \underbrace{\left\{ m^{t-1} - \nabla f(x^{t-1}) \right\}}_{=: a^{t}}
                \nonumber
                \\&\qquad\quad+ 
                (1 - \beta) e^t (e^t)^T \underbrace{\left\{ \widetilde{\nabla} f(x^t, \xi^t) - \nabla f(x^t) \right\}}_{=: b^{t}}
                \label{eq:jaguar_app_1}
                \\&\qquad\quad- 
                \left\{I - (1-\beta)e^t (e^t)^T \right\} \underbrace{\left\{ \nabla f(x^t) - \nabla f(x^{t-1}) \right\}}_{=:c^t}
            }_2^2}, \nonumber
        \end{align}
        where we used a notation $\widetilde{\nabla} f(x, \xi) := \sum_{i=1}^d \frac{\hat{f}(x + \tau e^i, \xi) - \hat{f}(x - \tau e^i, \xi)}{2 \tau} e^i$, and $e^i$ is the one-hot vector with $1$ in the $i$-th coordinate. In equation \eqref{eq:jaguar_app_1} we also used the classical notation of the identity matrix $I \in \mathbb{R}^{d \times d}$.

        Now using axillary notations $a^t, b^t, c^t$ from equation \eqref{eq:jaguar_app_1}, we divide it into six parts:
        \begin{equation}
        \label{eq:jaguar_app_2}
        \begin{split}
            \expect{\norm{a^{t+1}}_2^2} 
            &=
            \underbrace{\expect{\norm{
                \left\{I - (1-\beta)e^t (e^t)^T \right\} a^t}_2^2}}_{\circledOne}
            \\&+
            \underbrace{\expect{\norm{
                (1 - \beta) e^t (e^t)^T b^t}_2^2}}_{\circledTwo}
            \\&+
            \underbrace{\expect{\norm{
                \left\{I - (1-\beta)e^t (e^t)^T \right\} c^t}_2^2}}_{\circledThree}
            \\&+
            \underbrace{\expect{2\dotprod{
                \left\{I - (1-\beta)e^t (e^t)^T \right\} a^t
            }{
                (1 - \beta) e^t (e^t)^T b^t
            }}}_{\circledFour}
            \\&+
            \underbrace{\expect{2\dotprod{
                \left\{I - (1-\beta)e^t (e^t)^T \right\} a^t
            }{
                \left\{I - (1-\beta)e^t (e^t)^T \right\} c^t
            }}}_{\circledFive}
            \\&+
            \underbrace{\expect{2\dotprod{
                (1 - \beta) e^t (e^t)^T b^t
            }{
                \left\{I - (1-\beta)e^t (e^t)^T \right\} c^t
            }}}_{\circledSix}.
        \end{split}
        \end{equation}

        Consider $\circledOne$. Since $i_t$ from Algorithm \ref{algorithm:jaguar} is generated independent and uniform and $\{m^{s-1}, x^s\}_{s=0}^{t}$ do not depend on $i_t$, we can apply tower property:
        \begin{align}
        \label{eq:jaguar_circ1}
            \circledOne &= \expect{\norm{
                \left\{I - (1-\beta)e^t (e^t)^T \right\} a^t
            }_2^2}
            \nonumber
            \\&=
            \expect{
                (a^t)^T \left\{I - (1-\beta)e^t (e^t)^T \right\}^T \left\{I - (1-\beta)e^t (e^t)^T \right\} a^t
            }
            \nonumber
            \\&=
            \expect{
                (a^t)^T \left\{I - (1-\beta) (2 - (1-\beta)) e^t (e^t)^T \right\} a^t
            }
            \nonumber
            \\&=
            \expect{
                (a^t)^T \cdot \EEb{i_t \sim U[1; d]}{I - (1-\beta^2) e^t (e^t)^T} \cdot  a^t
            }
            \nonumber
            \\&=
            \expect{
                (a^t)^T \cdot \left(1 - \frac{1 - \beta^2}{d} \right) I \cdot  a^t
            }
            =
            \left(1 - \frac{1 - \beta^2}{d} \right) \expect{\norm{a^t}_2^2}.
        \end{align}
        Here we used the fact that $\left(e^t (e^t)^T\right)^T e^t (e^t)^T = e^t (e^t)^T$ and $\EEb{i_t \sim U[1; d]}{e^t (e^t)^T} = \frac{1}{d} I$.

        Similarly to equation \eqref{eq:jaguar_circ1}, we can estimate $\circledTwo$ and $\circledThree$: 
        \begin{align*}
            &\circledTwo = \expect{\norm{
                (1 - \beta) e^t (e^t)^T b^t
            }_2^2}
            =
            \frac{(1-\beta)^2}{d} \expect{\norm{b^t}^2},
            \\&
            \circledThree = \expect{\norm{
                \left\{I - (1-\beta)e^t (e^t)^T \right\} c^t}_2^2}
            =
            \left(1 - \frac{1 - \beta^2}{d} \right) \expect{\norm{c^t}^2} .
        \end{align*}
        Since $b^t = \widetilde{\nabla} f(x^t, \xi^t) - \nabla f(x^t)$, we can use Lemma 4 from \cite{veprikov2024new} with $\sigma_f = 0, \sigma_\nabla = \sigma$ and obtain the result of the form:
        \begin{align}
        \label{eq:jaguar_circ2}
            \circledTwo
            \leq
            \frac{(1-\beta)^2}{d} 
            \cdot
            \left( 
                d L^2 \tau^2 
                + 2 d \sigma^2 + \frac{2 d \Delta^2}{\tau^2}
            \right),
        \end{align}
        where $L, \sigma$ and $\Delta$ come from Assumptions \ref{as:lip}, \ref{as:sigma} and \ref{as:delta}. 
        
        Since $c^t = \nabla f(x^t) - \nabla f(x^{t-1})$, we can use Assumption \ref{as:lip} and obtain:
        \begin{align}
            \circledThree
            &\leq
            \left(1 - \frac{1 - \beta^2}{d} \right) L^2 \norm{x^t - x^{t-1}}_2^2 
            =
            \left(1 - \frac{1 - \beta^2}{d} \right) L^2 \norm{\text{sign} (m^t)}_2^2 
            \nonumber
            \\&= 
            \label{eq:jaguar_circ3}
            \left(1 - \frac{1 - \beta^2}{d} \right) d L^2 \gamma^2
            \leq
            d L^2 \gamma^2 .
        \end{align}

        Consider $\circledFour$. Let us move all matrixes to the left side of the dot product:
        \begin{align*}
            \circledFour 
            &=
            \expect{2\dotprod{
                (1-\beta)\left\{I - (1-\beta)e^t (e^t)^T \right\} e^t (e^t)^T \cdot a^t
            }{
                b^t
            }}
            \\&=
            \expect{2\dotprod{
                (1-\beta) \beta e^t (e^t)^T \cdot a^t
            }{
                b^t
            }}.
        \end{align*}
        Now we use tower property for $i_t$ as we did for $\circledOne, \circledTwo,\circledThree$ and use the definitions of $a^t$ and $b^t$:
        \begin{align*}
            \circledFour 
            &=
            \frac{(1-\beta) \beta}{d} \cdot \expect{2\dotprod{
                a^t
            }{
                b^t
            }}
            \\&=
            \frac{(1-\beta) \beta}{d} \cdot \expect{2\dotprod{
                m^{t-1} - \nabla f(x^{t-1}) 
            }{
                \widetilde{\nabla} f(x^t, \xi^t) - \nabla f(x^t)
            }}.
        \end{align*}
        We now again use tower property, but with stochastic variable $\xi^t$. Since $\{m^{s-1}, x^s\}_{s=0}^t$ do not depend on $\xi^t$, we can obtain that:
        \begin{align}
            \circledFour 
            &=
            \frac{(1-\beta) \beta}{d} \cdot \expect{2\dotprod{
                m^{t-1} - \nabla f(x^{t-1}) 
            }{
                \EEb{\xi^t}{\widetilde{\nabla} f(x^t, \xi^t)} - \nabla f(x^t)
            }}
            \nonumber
            \\&\leq
            \frac{(1-\beta) \beta}{2d} \cdot \expect{\norm{m^{t-1} - \nabla f(x^{t-1})}_2^2}
            \label{eq:jaguar_tmp_3}
            \\&~\quad+
            \frac{2(1-\beta) \beta}{d} \cdot \expect{\norm{\EEb{\xi^t}{\widetilde{\nabla} f(x^t, \xi^t)} - \nabla f(x^t)}_2^2}.
            \nonumber
        \end{align}
        In \eqref{eq:jaguar_tmp_3} we use Fenchel-Young inequality. For estimating $\|\mathbb{E}_{\xi^t}[\widetilde{\nabla} f(x^t, \xi^t)] - \nabla f(x^t)\|_2^2$ we again can use Lemma 4 from \cite{veprikov2024new} but now with $\sigma_\nabla = \sigma_f = 0$ since we have no randomness in $\EEb{\xi^t}{\widetilde{\nabla} f(x^t, \xi^t)}$. Therefore $\circledFour$ is bounded as:
        \begin{align}
        \label{eq:jaguar_circ4}
            \circledFour 
            &\leq
            \frac{(1-\beta) \beta}{2d} \cdot \expect{\norm{a^t}_2^2}
            +
            \frac{2(1-\beta) \beta}{d} \cdot \left( 
                d L^2 \tau^2 
                + \frac{2 d \Delta^2}{\tau^2} 
            \right).
        \end{align}
        Consider $\circledFive$. Similar to $\circledFour$ we can obtain:
        \begin{align}
            \circledFive 
            &= 
            \expect{2\dotprod{
                \left\{I - (1-\beta)e^t (e^t)^T \right\} a^t
            }{
                \left\{I - (1-\beta)e^t (e^t)^T \right\} c^t
            }}
            \nonumber
            \\&=
            \expect{2\dotprod{
                \left\{I - (1-\beta^2)e^t (e^t)^T \right\} a^t
            }{
                c^t
            }}
            \nonumber
            \\&=
            \left(1 - \frac{1 - \beta^2}{d} \right) \cdot \expect{2\dotprod{
                a^t
            }{
                c^t
            }}
            \nonumber
            \\&\leq
            \left(1 - \frac{1 - \beta^2}{d} \right) \cdot \frac{1 - \beta}{2d} \cdot \expect{\norm{a^t}_2^2}
            +
            \left(1 - \frac{1 - \beta^2}{d} \right) \cdot \frac{2d}{1 - \beta} \cdot \expect{\norm{c^t}_2^2}
            \nonumber
            \\&\leq
            \frac{1 - \beta}{2d} \cdot \expect{\norm{a^t}_2^2}
            +
            \frac{2d}{1 - \beta} \cdot d L^2 \gamma^2 .
            \label{eq:jaguar_circ5}
        \end{align}
        Finally, we estimate $\circledSix$ in the same way:
        \begin{align}
            \circledSix 
            &=
            \expect{2\dotprod{
                (1 - \beta) e^t (e^t)^T b^t
            }{
                \left\{I - (1-\beta)e^t (e^t)^T \right\} c^t
            }}
            \nonumber
            \\&=
            \expect{2\dotprod{
                (1 - \beta) \beta e^t (e^t)^T b^t
            }{
                c^t
            }}
            \nonumber
            \\&=
            \frac{(1 - \beta) \beta}{d} \cdot \expect{2\dotprod{
                b^t
            }{
                c^t
            }}
            \nonumber
            \\&\leq
            \frac{(1 - \beta) \beta}{d} \cdot 
            \expect{\norm{\EEb{\xi^t}{\widetilde{\nabla} f(x^t, \xi^t)} - \nabla f(x^t)}_2^2}
            +
            \frac{(1 - \beta) \beta}{d} \cdot \expect{\norm{c^t}_2^2}
            \nonumber
            \\&\leq
            \frac{(1 - \beta) \beta}{d} \cdot \left( 
                d L^2 \tau^2 
                + \frac{2 d \Delta^2}{\tau^2} 
            \right)
            +
            \frac{(1 - \beta) \beta}{d} \cdot dL^2\gamma^2 .
            \label{eq:jaguar_circ6}
        \end{align}
        We made it! Now let us combine equations \eqref{eq:jaguar_circ1}, \eqref{eq:jaguar_circ2}, \eqref{eq:jaguar_circ3}, \eqref{eq:jaguar_circ4}, \eqref{eq:jaguar_circ5} and \eqref{eq:jaguar_circ6} to bound $\mathbb{E}[\|a^{t+1}\|_2^2]$ from equation \eqref{eq:jaguar_app_2}:
        \begin{align*}
            \expect{\norm{a^{t+1}}_2^2}
            &\leq
            \left( 1 - \frac{1 - \beta}{d} \left[ \underbrace{1 + \beta}_{\eqref{eq:jaguar_circ1}} - \underbrace{\frac{\beta}{2}}_{\eqref{eq:jaguar_circ4}} - \underbrace{\frac{1}{2}}_{\eqref{eq:jaguar_circ5}} \right] \right) \cdot \expect{\norm{a^t}_2^2}
            \\&~~~~+
            \frac{1-\beta}{d} \left( \underbrace{1 - \beta}_{\eqref{eq:jaguar_circ2}} + \underbrace{2\beta}_{\eqref{eq:jaguar_circ4}} + \underbrace{\beta}_{\eqref{eq:jaguar_circ6}} \right) \cdot \left( 
                d L^2 \tau^2 
                + \frac{2 d \Delta^2}{\tau^2} 
            \right)
            +
            \underbrace{\frac{(1-\beta)^2}{d}}_{\eqref{eq:jaguar_circ2}} \cdot 2 d \sigma^2
            \\&~~~~+
            \left(\underbrace{1}_{\eqref{eq:jaguar_circ3}} + \underbrace{\frac{2d}{1-\beta}}_{\eqref{eq:jaguar_circ5}} + \underbrace{\frac{(1-\beta)\beta}{d}}_{\eqref{eq:jaguar_circ6}} \right) \cdot dL^2 \gamma^2
            \\&\leq
            \left(1 - \frac{1 - \beta^2}{2d} \right) \cdot \expect{\norm{a^t}_2^2}
            \\&~~~~+
            3 \frac{1 - \beta}{d} \cdot \left( 
                d L^2 \tau^2 
                + \frac{2 d \Delta^2}{\tau^2} 
            \right)
            +
            2 \frac{(1-\beta)^2}{d} \cdot  d \sigma^2
            +
            \frac{4 d}{1 - \beta} \cdot d L^2 \gamma^2 .
        \end{align*}
        By unrolling the recursion in the last inequality we obtain:
        \begin{align*}
            \expect{\norm{m^t - \nabla f(x^t)}_2^2}
            &\leq
            8 \frac{d^2}{(1 - \beta)(1 - \beta^2)} \cdot d L^2 \gamma^2 
            +
            4 \frac{(1-\beta)^2}{1 - \beta^2} \cdot  d \sigma^2
            \\&~~~~+
            6 \frac{1 - \beta}{1 - \beta^2} \cdot \left( 
                d L^2 \tau^2 
                + \frac{2 d \Delta^2}{\tau^2} 
            \right)
            +
            \left( \frac{1 - \beta^2}{2d} \right)^t \norm{\nabla f(x^0)}_2^2
            \\&=
            \mathcal{O}\Bigg[ 
                \frac{d^3}{(1-\beta)^2} L^2 \gamma^2  
                +
                (1-\beta)d \sigma^2
                +
                d L^2 \tau^2 
                +
                \frac{2 d \Delta^2}{\tau^2} 
                \\&~~~~~~~~~~+ 
                \left( 1 - \frac{1 - \beta}{2d} \right)^t \norm{\nabla f(x^0)}_2^2
            \Bigg] .
        \end{align*}
        This finishes the proof.
    \end{proof}

    \subsection{Proof of Theorem \ref{theorem:jaguar_sign}}
    \begin{proof}
        We start from using Lemma 1 from \cite{sun2023momentum}. For the points $x^t$, generated by Algorithm \ref{algorithm:jaguar} it holds that:
        \begin{align}
        \label{eq:step_lemma_sign_sgd}
            f(x^{t+1}) - f(x^t) 
            \leq
            - \gamma \norm{\nabla f(x^t)}_1
            +
            2 \sqrt{d} \gamma \norm{m^t - \nabla f(x^t)}_2
            +
            \frac{d L \gamma^2}{2} .
        \end{align}
        Now we take mathematical expectation of the both sides of the inequality \eqref{eq:step_lemma_sign_sgd} and use the results from Lemma \ref{lemma:mt_jaguar}:
        \begin{align*}
            \expect{f(x^{t+1})} - \expect{f(x^t)}
            &\leq
            - \gamma \expect{\norm{\nabla f(x^t)}_1}
            +
            2 \sqrt{d} \gamma \expect{\norm{m^t - \nabla f(x^t)}_2}
            +
            \frac{d L \gamma^2}{2}
            \\&=
            - \gamma \expect{\norm{\nabla f(x^t)}_1}
            +
            \mathcal{O} \Bigg[ 
                \frac{d^2}{1-\beta}\cdot L \gamma^2  
                +
                \sqrt{1-\beta}d \gamma \sigma
                +
                d \gamma L \tau
                \\&~~~~~+
                \frac{d \gamma \Delta}{\tau} 
                + 
                \sqrt{d} \gamma \left( 1 - \frac{1 - \beta}{d} \right)^{t/2} \norm{\nabla f(x^0)}_2
            \Bigg] +
            \frac{d L \gamma^2}{2} .
        \end{align*}
        Consequently, after summing all $T$ steps, we obtain:
        \begin{equation}
        \label{eq:tmp_1111}
        \begin{split}
            \gamma \sum_{t=0}^T \expect{\norm{\nabla f(x^t)}_1}
            =
            \mathcal{O} \Bigg[ 
                &f(x^0) - f(x^T)
                +
                T \cdot \left( \frac{d^2}{1-\beta}\cdot L \gamma^2  
                +
                \sqrt{1-\beta}d \gamma \sigma
                +
                d \gamma L \tau
                \right)
                \\&+
                T \cdot \frac{d \gamma \Delta}{\tau} 
                + 
                \sqrt{d} \gamma \sum_{t=0}^T \left( 1 - \frac{1 - \beta}{d} \right)^{t/2} \norm{\nabla f(x^0)}_2
            \Bigg] .
        \end{split}
        \end{equation}
        Now, we divide equation \eqref{eq:tmp_1111} by $\gamma T$ from both sides and obtain:
        \begin{equation*}
        \begin{split}
            \frac{1}{T} \sum_{t=0}^T \expect{\norm{\nabla f(x^t)}_1}
            =
            \mathcal{O} \Bigg[ 
                &\frac{\delta_0}{\gamma T}
                +
                \frac{d \norm{\nabla f(x^0)}_2}{T \sqrt{1 - \beta}}
                +
                \frac{d^2 L \gamma}{1-\beta}  
                +
                \sqrt{1-\beta}d \sigma
                +
                d L \tau
                +
                \frac{d \Delta}{\tau} 
            \Bigg],
        \end{split}
        \end{equation*}
        where we used a notation $\delta_0 := f(x^0) - f^*$. This finishes the proof.
    \end{proof}

\section{Proofs for ZO Muon with JAGUAR (Algorithm \ref{algorithm:muon})}

    \subsection{Technical Lemmas}
    \begin{lemma}
    \label{lemma:A_dotprod_U}
        Consider two arbitrary matrixes $A, B$ of the same shape and their SVD decomposition: $A = U_A \Sigma_A V_A^T$, $B = U_B \Sigma_B V_B^T$. Define $r_A$ and $r_B$ as ranks of $A$ and $B$, then it holds that
        \begin{align*}
            \left| \dotprod{A}{U_A V_A^T - U_B V_B^T} \right|
            \leq
            2\norm{A - B}_{\mathcal{S}_1} 
            \leq
            2\sqrt{\operatorname{rank}(A-B)} \norm{A - B}_F .
        \end{align*}
    \end{lemma}
    \begin{proof}
        We first provide an axillary notation:
        \begin{equation*}
            \delta := \dotprod{A}{U_A V_A^T - U_B V_B^T}.
        \end{equation*}
        Because $U_A$ and $V_A$ have orthonormal columns:  
        \[
        \langle A,\,U_A V_A^\top\rangle
        = \operatorname{tr}\!\bigl(V_A \Sigma_A U_A^\top U_A V_A^\top\bigr)
        = \operatorname{tr}(\Sigma_A)
        = \|A\|_{\mathcal{S}_1}.
        \]
        Hence  
        \[
        \delta
        = \|A\|_{\mathcal{S}_1} - \langle A,\,U_B V_B^\top\rangle.
        \]
    
        Insert $B$ and regroup:
        \[
        \delta
        = \|A\|_{\mathcal{S}_1}
           - \bigl(\langle B,\,U_B V_B^\top\rangle
                  + \langle A-B,\,U_B V_B^\top\rangle\bigr)
        = \|A\|_{\mathcal{S}_1} - \|B\|_{\mathcal{S}_1}
          - \langle A-B,\,U_B V_B^\top\rangle .
        \]
    
        The first difference is controlled by the triangle inequality for the nuclear norm:
        \[
        \bigl|\|A\|_{\mathcal{S}_1} - \|B\|_{\mathcal{S}_1}\bigr|
        \le \|A-B\|_{\mathcal{S}_1}.
        \]
        For the second term, Hölder’s inequality with $\|U_B V_B^\top\|_2 = 1$ gives
        \[
        \bigl|\langle A-B,\,U_B V_B^\top\rangle\bigr|
        \le \|A-B\|_{\mathcal{S}_1}.
        \]
        
        Therefore
        \[
        |\delta|
        \le \|A-B\|_{\mathcal{S}_1} + \|A-B\|_{\mathcal{S}_1}
        = 2\,\|A-B\|_{\mathcal{S}_1}.
        \]
        
        Using the connection between the Frobenius ($\mathcal{S}_2$) by nuclear ($\mathcal{S}_1$) norms we obtain that:
        \[
        |\delta|
        =
        \dotprod{A}{U_A V_A^T - U_B V_B^T}
        \le 2\,\|A-B\|_{\mathcal{S}_1}
        \le 2\sqrt{\operatorname{rank}(A-B)}\,\|A-B\|_F .
        \]
        The factor $2$ in the nuclear norm bound is sharp, as equality holds for $B=-A$. This finishes the proof.
    \end{proof}

    We now provide lemma similar to the step Lemma 1 from \cite{sun2023momentum}, but in the matrix case.
    \begin{lemma}[Step lemma for Muon with momentum]
    \label{lemma:muon_step}
        Let $f$ be an $L$-smooth function (Assumption \ref{as:lip}), and let  $X^\dagger, M \in \mathbb{R}^{m \times n}$ with $m \geq n$ be an arbitrary matrixes. We define
        \begin{equation*}
            X^\ddagger := X^\dagger - \gamma \cdot U_M V_M^T,
        \end{equation*}
        where $\gamma > 0$ and $U_M V_M^T$ comes from SVD decomposition of $M$: $M = U_M \Sigma_M V_M^T$. Then, it holds that:
        \begin{align*}
            f\left( X^\ddagger \right) - f\left( X^\dagger \right)
            \leq
            -\gamma \norm{\nabla f\left( X^\dagger \right)}_{\mathcal{S}_1}
            +
            2 \sqrt{n}\gamma \norm{\nabla f\left( X^\dagger \right) - M}_{F}
            +
            \frac{L n \gamma^2}{2}.
        \end{align*}
    \end{lemma}
    \begin{proof}
        The $L$-smoothness of the gradient (Assumption \ref{as:lip}) gives us 
        \begin{align*}
            f\left( X^\ddagger \right) - f\left( X^\dagger \right)
            &\leq
            \dotprod{\nabla f\left( X^\dagger \right)}{X^\ddagger - X^\dagger}
            +
            \frac{L}{2} \norm{X^\ddagger - X^\dagger}_F^2
            \\&=
            -\gamma \dotprod{\nabla f\left( X^\dagger \right)}{U_M V_M^T}
            +
            \frac{L n \gamma^2}{2}
            \\&=
            -\gamma \dotprod{\nabla f\left( X^\dagger \right)}{U_\nabla V_\nabla^T}
            +
            \gamma \dotprod{\nabla f\left( X^\dagger \right)}{U_\nabla V_\nabla^T - U_M V_M^T}
            +
            \frac{L n \gamma^2}{2},
        \end{align*}
        where $U_\nabla V_\nabla^T$ comes from SVD decomposition of $\nabla f\left( X^\dagger \right)$: $\nabla f\left( X^\dagger \right) = U_\nabla \Sigma_\nabla V_\nabla^T$. Therefore the first dot product takes form:
        \begin{align*}
            -\gamma \dotprod{\nabla f\left( X^\dagger \right)}{U_\nabla V_\nabla^T}
            &=
            -\gamma \text{tr} \left( V_\nabla \Sigma_\nabla U_\nabla^T U_\nabla V_\nabla^T\right)
            =
            -\gamma \text{tr} \left( \Sigma_\nabla\right)
            =
            - \gamma \norm{\nabla f\left( X^\dagger \right)}_{\mathcal{S}_1} .
        \end{align*}

        Now we utilize Lemma \ref{lemma:A_dotprod_U} with $A = \nabla f\left( X^\dagger \right)$ and $B = M$:
        \begin{align*}
            f\left( X^\ddagger \right) - f\left( X^\dagger \right)
            &\leq
            -\gamma \norm{\nabla f\left( X^\dagger \right)}_{\mathcal{S}_1}
            +
            2 \gamma \norm{\nabla f\left( X^\dagger \right) - M}_{\mathcal{S}_1}
            +
            \frac{L n \gamma^2}{2}
            \\&\leq
            -\gamma \norm{\nabla f\left( X^\dagger \right)}_{\mathcal{S}_1}
            +
            2 \sqrt{n}\gamma \norm{\nabla f\left( X^\dagger \right) - M}_{F}
            +
            \frac{L n \gamma^2}{2}.
        \end{align*}
        This finishes the proof.
    \end{proof}

\subsection{Proof of Theorem \ref{theorem:jaguar_muon}}

    \begin{proof}
        We start from using Lemma \ref{lemma:muon_step}. For the points $X^t$, generated by Algorithm \ref{algorithm:muon} it holds that:
        \begin{align}
        \label{eq:muon_tmp_1}
            f\left( X^{t+1} \right) - f\left( X^t \right)
            &\leq
            -\gamma \norm{\nabla f\left( X^t \right)}_{\mathcal{S}_1}
            +
            2 \sqrt{n}\gamma \norm{\nabla f\left( X^t \right) - M^t}_{F}
            +
            \frac{L n \gamma^2}{2}.
        \end{align}
        Now we take mathematical expectation of the both sides if \eqref{eq:muon_tmp_1} and bound the term $\mathbb{E}[\|\nabla f\left( X^t \right) - M^t\|_{F}]$ we again use Lemma \ref{lemma:mt_jaguar} with $x^t = \overline{\text{vec}}(X^t)$ and $ m^t = \overline{\text{vec}}(M^t)$. The result of Lemma \ref{lemma:mt_jaguar} holds true with $d = m \cdot n$, since $\norm{A}_F = \norm{\overline{\text{vec}}(A)}_2$. Therefore \eqref{eq:muon_tmp_1} takes form:
        \begin{align*}
            \expect{f(X^{t+1})} - \expect{f(X^t)}
            &\leq
            - \gamma \expect{\norm{\nabla f(X^t)}_{\mathcal{S}_1}}
            +
            2 \sqrt{n} \gamma \expect{\norm{M^t - \nabla f(X^t)}_2}
            +
            \frac{n L \gamma^2}{2}
            \\&=
            - \gamma \expect{\norm{\nabla f(X^t)}_{\mathcal{S}_1}}
            +
            n^{1/2}\mathcal{O} \Bigg[ 
                \frac{(mn)^{3/2}}{1-\beta}\cdot L \gamma^2  
                \\&~~~~~+
                \sqrt{1-\beta}(mn)^{1/2} \gamma \sigma
                +
                (mn)^{1/2} \gamma L \tau
                +
                \frac{(mn)^{1/2} \gamma \Delta}{\tau} 
                \\&~~~~~+
                n^{1/2} \gamma \left( 1 - \frac{1 - \beta}{mn} \right)^{t/2} \norm{\nabla f(X^0)}_2
            \Bigg] +
            \frac{n L \gamma^2}{2} .
        \\&=
        - \gamma \expect{\norm{\nabla f(X^t)}_{\mathcal{S}_1}}
        +
        \mathcal{O} \Bigg[  
            \frac{m^{3/2}n^{2}}{1-\beta}\cdot L \gamma^2
            \\&~~~~~+
            \sqrt{1-\beta} m^{1/2}n \gamma \sigma
            +
            m^{1/2}n \gamma L \tau
            +
            \frac{m^{1/2}n \gamma \Delta}{\tau} 
            \\&~~~~~+
            n^{1/2} \gamma \left( 1 - \frac{1 - \beta}{mn} \right)^{t/2} \norm{\nabla f(X^0)}_2
        \Bigg] .
        \end{align*}
        Consequently, after summing all $T$ steps, we obtain:
        \begin{equation}
        \label{eq:tmp_222222}
        \begin{split}
            \gamma \sum_{t=0}^T \expect{\norm{\nabla f(X^t)}_{\mathcal{S}_1}}
            =
            \mathcal{O} \Bigg[ 
                &f(X^0) - f(X^T)
                \\&+
                T \cdot \left( \frac{m^{3/2}n^{2}}{1-\beta}\cdot L \gamma^2  
                +
                \sqrt{1-\beta} m^{1/2}n \gamma \sigma
                \right)
                \\&+
                T \cdot \left(
                m^{1/2}n \gamma L \tau
                +
                \frac{m^{1/2}n \gamma \Delta}{\tau} 
                \right)
                \\&+
                n^{1/2} \gamma \sum_{t=0}^T \left( 1 - \frac{1 - \beta}{mn} \right)^{t/2} \norm{\nabla f(X^0)}_2
            \Bigg] .
        \end{split}
        \end{equation}
        Now, we divide equation \eqref{eq:tmp_222222} by $\gamma T$ from both sides and obtain:
        \begin{equation*}
        \begin{split}
            \frac{1}{T} \sum_{t=0}^T \expect{\norm{\nabla f(X^t)}_{\mathcal{S}_1}}
            =
            \mathcal{O} \Bigg[ 
                &\frac{\delta_0}{\gamma T}
                +
                \frac{m^{1/2}n \norm{\nabla f(x^0)}_2}{T \sqrt{1 - \beta}}
                +
                \frac{m^{3/2}n^2 \gamma}{1-\beta}  
                +
                \sqrt{1-\beta} m^{1/2}n \sigma
                \\&+
                m^{1/2}n L \tau
                +
                \frac{m^{1/2}n \Delta}{\tau} 
            \Bigg],
        \end{split}
        \end{equation*}
        where we used a notation $\delta_0 := f(x^0) - f^*$. This finishes the proof.
    \end{proof}

\end{appendixpart}
\end{document}